\documentclass[11pt]{article}

\usepackage{amsthm}
\usepackage{amssymb}
\usepackage{amsmath}
\usepackage{graphicx}
\usepackage{algorithm}
\usepackage{algorithmic}
 \usepackage{tcolorbox}
 \usepackage{geometry}
 \usepackage[titletoc,title]{appendix}

\newtheorem{theorem}{Theorem}
\newtheorem{lemma}{Lemma}
\newtheorem{definition}{Definition}

\newtheorem{assumption}{Assumption}

\usepackage[hidelinks]{hyperref}
\usepackage[numbers,sort&compress]{natbib}

\newcommand{\longversion}[1]{}

\renewcommand{\longversion}[1]{}

\begin{document}

\title{Large Scale Markov Decision Processes with Changing Rewards}
\author{Adrian Rivera Cardoso \thanks{School of Industrial and Systems Engineering, Georgia Institute of Technology. {\tt adrian.riv@gatech.edu}. } \and  He Wang \thanks{School of Industrial and Systems Engineering, Georgia Institute of Technology. {\tt he.wang@isye.gatech.edu}.} \and  Huan Xu \thanks{School of Industrial and Systems Engineering, Georgia Institute of Technology. {\tt huan.xu@isye.gatech.edu}.}}
 \maketitle

\begin{abstract}
We consider Markov Decision Processes (MDPs) where the rewards are unknown and may change in an adversarial manner.  We provide an algorithm that achieves state-of-the-art regret bound of $O( \sqrt{\tau (\ln|S|+\ln|A|)T}\ln(T))$, where $S$ is the state space, $A$ is the action space, $\tau$ is the mixing time of the MDP, and $T$ is the number of periods. The algorithm's computational complexity  is polynomial in $|S|$ and $|A|$ per period. We then consider a setting often encountered in practice, where the state space of the MDP is too large to allow for exact solutions. By approximating the state-action occupancy measures with a linear architecture of dimension $d\ll|S|$, we propose a modified algorithm with computational complexity polynomial in $d$. We also prove a regret bound for this modified algorithm, which to the best of our knowledge this is the first $\tilde{O}(\sqrt{T})$ regret bound for large scale MDPs with changing rewards.
\end{abstract}

\section{Introduction}

In this paper, we study Markov Decision Processes (hereafter MDPs) with arbitrarily varying rewards.
MDP provides a general mathematical framework for modeling sequential decision making under uncertainty \cite{bertsekas1995dynamic, howard1960dynamic, puterman2014markov}. In the standard MDP setting, if the process is in some state $s$, the decision maker takes an action $a$ and receives an expected reward $r(s,a)$, before the process randomly transitions into a new state. The goal of the decision maker is to maximize the total expected reward. It is assumed that the decision maker has complete knowledge of the reward function $r(s,a)$, which does not change over time.

Over the past two decades, there has been much interest in sequential learning and decision making in an unknown and possibly \emph{adversarial} environment.
A wide range of sequential learning problems can be modeled using the framework of Online Convex Optimization (OCO) \cite{zinkevich2003online,hazan2016introduction}. In OCO, the decision maker plays a repeated game against an adversary for a given number of rounds. At the beginning of each round indexed by $t$, the decision maker chooses an action $a_t$ in some convex compact set $A$ and the adversary chooses a concave reward function $r_t$, hence a reward of $r_t(a_t)$ is received. After observing the realized reward function, the decision maker chooses its next action $a_{t+1}$ and so on. Since the decision maker does not know how the future reward functions will be chosen, its goal is to achieve a small \emph{regret}; that is, the cumulative reward earned throughout the game should be close to the cumulative reward if the decision maker had been given the benefit of hindsight to choose one fixed action. We can express the regret after $T$ rounds as
\[
\text{Regret} (T) = \max_{a \in A} \sum_{t=1}^T r_t(a) - \sum_{t=1}^T r_t(a_t).
\]
The OCO model has many applications such as universal portfolios \cite{cover1991, kalai2002, helmbold1998}, online shortest path \cite{takimoto2003path}, and online submodular minimization \cite{hazan2012submodular}. It also has close relations with areas such as convex optimization \cite{hazan2010optimal, ben2015oracle} and game theory \cite{cesa2006prediction}. There are many algorithms that guarantee sublinear regret, e.g., Online Gradient Descent \cite{zinkevich2003online}, Perturbed Follow the Leader \cite{kalai2005efficient}, and Regularized Follow the Leader \cite{shalev2007online,abernethy2009competing}.
Compared with the MDP setting, the main difference is that in OCO there is no notion of states, however the payoffs may be chosen by an adversary.

In this work, we study a general problem that unites the MDP and the OCO frameworks, which we call the {\bf Online MDP problem}. More specifically, we consider MDPs where the decision maker knows the transition probabilities but the rewards are dynamically chosen by an adversary. 
The Online MDP model can be used for a wide range of applications, including multi-armed bandits with constraints \cite{yu2009markov}, the paging problem in computer operating systems \cite{even2009online}, the $k$-server problem  \cite{even2009online}, stochastic inventory control in operations research \cite{puterman2014markov}, and scheduling of queueing networks \cite{de2003linear,abbasi2014linear}.

\subsection{Main Results}
We propose a new computationally efficient algorithm that achieves near optimal regret for the Online MDP problem.
Our algorithm is based on the linear programming formulation of infinite-horizon average reward MDPs, which uses the occupancy measure of state-action pairs as decision variables. This approach differs from other papers that have studied the Online MDP problem previously, see review in \S\ref{subsec:literature}. 

We prove that the algorithm achieves regret bounded by $O(\tau +\sqrt{\tau T  (\ln \vert S \vert +\ln \vert A \vert)} \ln(T) )$, where $S$ denotes the state space, $A$ denotes the action space, $\tau$ is the mixing time of the MDP, and $T$ is the number of periods. Notice that this regret bound depends \emph{logarithmically} on the size of state and action space.
The algorithm solves a regularized linear program in each period with $poly(|S||A|)$ complexity. The regret bound and the computation complexity compares favorably to the existing methods discussed in \S\ref{subsec:literature}. 


We then extend our results to the case where the state space $S$ is extremely large so that $poly(|S||A|)$ computational complexity is impractical. We assume the state-action occupancy measures associated with stationary policies are approximated with a linear architecture of dimension $d \ll |S|$.
We design an approximate algorithm combining several innovative techniques for solving large scale MDPs inspired by \cite{abbasi2019large,abbasi2014linear}.
A salient feature of this algorithm is that its computational complexity does not depend on the size of the state-space but instead on the number of features $d$.
The algorithm has a regret bound $O(c_{S,A}(\ln|S|+\ln|A|)\sqrt{\tau T}\ln T)$, where $c_{S,A}$ is a problem dependent constant.
To the best of our knowledge, this is the first $\tilde{O}(\sqrt{T})$ regret result for large scale Online MDPs.

\subsection{Related Work}
\label{subsec:literature}

The history of MDPs goes back to the seminal work of Bellman \cite{bellman1957markovian} and Howard \cite{howard1960dynamic} from the 1950's. 
Some classic algorithms for solving MDPS include policy iteration, value iteration, policy gradient, Q-learning and their approximate versions (see \cite{puterman2014markov, bertsekas1995dynamic,bertsekas1996neuro} for an excellent discussion). In this paper, we will focus on a relatively less used approach, which is based on finding the \textit{occupancy measure} using linear programming, as done recently in \cite{chen2018scalable,wang2017primal,abbasi2019large} to solve MDPs with \emph{static}  rewards (see more details in Section \ref{mdp_via_lp}). To deal with the curse of dimensionality, \cite{chen2018scalable} uses bilinear functions to approximate the occupancy measures and \cite{abbasi2019large} uses a linear approximation.

The Online MDP problem was first studied a decade ago by \cite{yu2009markov,even2009online}.  In \cite{even2009online}, the authors developed no regret algorithms where the bound scales as $O(\tau^2 \sqrt{T \ln(\vert A \vert)})$, where $\tau$ is the mixing time (see \S\ref{sec:mdp_rftl}).  Their method runs an expert algorithm (e.g. Weighted Majority \cite{littlestone1994weighted}) on every state where the actions are the experts. However, the authors did not consider the case with large state space in their paper. 
 In \cite{yu2009markov}, the authors provide a more computationally efficient algorithm using a variant of Follow the Perturbed Leader \cite{kalai2005efficient}, but unfortunately their regret bound becomes $O(|S||A|^2\tau T^{3/4+\epsilon})$. 
They also considered approximation algorithm for large state space, but did not establish an exact regret bound.
 The work most closely related to ours is that from \cite{dick2014online}, where the authors also use a linear programming formulation of MDP similar to ours. 
However, there seem to be some gaps in the proof of their results.\footnote{In particular, we believe the proof of Lemma 1 in \cite{dick2014online} is incorrect. Equation (8) in their paper states that the regret relative to a policy is equal to the sum of a sequence of vector products; however, the dimensions of vectors involved in these dot products are incompatible. By their definition, the variable $\nu_t$ is a vector of dimension $\vert S \vert$, which is being multiplied with a loss vector with dimension $\vert S \vert \vert A \vert$.}

 The paper \cite{ma2015online} also considers Online MDPs with large state space.  Under some conditions,  they show sublinear regret using a variant of approximate policy iteration, but the regret rate is left unspecified in their paper. \cite{zimin2013online} considers a special class of MDPs called \textit{episodic} MDPs and design algorithms using the occupancy measure LP formulation. Following this line of work, \cite{neu2017unified} shows that several reinforcement learning algorithms can be viewed as variant of Mirror Descent \cite{juditsky2011first} thus one can establish convergence properties of these algorithms. In \cite{neu2014online} the authors consider Online MDPs with bandit feedback and provide an algorithm based on \cite{even2009online}'s with regret of $O(T^{2/3})$.
 
 A more general problem to the Online MDP setting considered here is where the MDP transition probabilites also change in an adversarial manner, which is beyond the scope of this paper. It is believed that this problem is much less tractable computationally \cite[see discussion in][]{even2005experts}. \cite{yu2009online} studies MDPs with changing transition probabilities, 
although \cite{neu2014online} questions the correctness of their result, as the regret obtained seems to have broken a lower bound. In \cite{gajane2018sliding}, the authors use a sliding window approach under a particular definition of regret. \cite{abbasi2013online} shows sublinear regret with changing transition probabilities when they compare against a restricted policy class.

\section{Problem Formulation: Online MDP}\label{section:online_mdps}
We consider a general Markov Decision Process with known transition probabilities but unknown and adversarially chosen rewards. Let $S$ denote the set of possible states, and $A$ denote the set of actions. (For notational simplicity, we assume the set of actions a player can take is the same for all states, but this assumption can be relaxed easily.)
At each period $t \in [T]$, if the system is in state $s_t \in S$, the decision maker chooses an action $a_t \in A$ and collects a reward $r_t(s_t,a_t)$. Here, $r_t : S \times A \rightarrow [-1,1]$ denotes a reward function for period $t$. 
We assume that the sequence of reward functions $\{r_t\}_{t=1}^T$ is initially unknown to the decision maker. The function $r_t$ is revealed only after the action $a_t$ has been chosen. We allow the sequence $\{r_t\}_{t=1}^T$ to be chosen by an \textit{adaptive adversary}, meaning $r_t$ can be chosen using the history $\{s_i\}_{i=1}^{t}$ and $\{a_i\}_{i=1}^{t-1}$; in particular, the adversary does \emph{not} observe the action $a_t$ when choosing $r_t$. 
After $a_t$ is chosen, the system then proceeds to state $s_{t+1}$ in the next period with probability $P(s_{t+1}\vert s_t, a_t)$. 
We assume the decision maker has complete knowledge of the transition probabilities given by $P(s' \vert s, a) : S \times A \rightarrow S$.

Suppose the initial state of the MDP follows $s_1 \sim \nu_1$, where $\nu_1$ is a probability distribution over  $S$. 
The objective of the decision maker is to choose a sequence of actions based on the history of states and rewards observed, such that the cumulative reward in $T$ periods is close to that of the optimal offline static policy.
Formally, let $\pi$ denote a stationary (randomized) policy: $\pi:S\rightarrow \Delta_A$, where $\Delta_A$ is the set of probability distributions over the action set $A$. Let $\Pi$ denote the set of all stationary policies. We aim to find an algorithm that minimizes  
\begin{align}\label{regret_def}
\text{MDP-Regret}(T)\triangleq \sup_{\pi \in \Pi} R(T,\pi), \; \text{with } R(T,\pi) \triangleq \mathbb{E}[\sum_{t=1}^T r_t(s^\pi_t , a^\pi_t)] - \mathbb{E}[\sum_{t=1}^T r_t(s_t,a_t)],
\end{align}
where the expectations are taken with respect to random transitions of MDP and (possibly) external randomization of the algorithm.

 \section{Preliminaries}


Next we provide additional notation for the MDP.  
Let $P^\pi_{s,s'} \triangleq P(s' \mid s, \pi(s))$ be the probability of transitioning from state $s$ to $s'$ given a policy $\pi$. Let $P^\pi$ be the $\vert S\vert \times \vert S\vert$ matrix with entries $P^\pi_{s,s'} \, \forall s,s' \in S$. 
We use row vector $\nu_t \in \Delta_S$ to denote the probability distribution over states at time $t$. Let $\nu^\pi_{t+1}$ be the distribution over states at time $t+1$ under policy $\pi$, given by $\nu^{\pi}_{t+1} = \nu_{t} P^\pi$.  
Let $\nu^\pi_{st}$ denote the stationary distribution for policy $\pi$, which satisfies the linear equation $\nu^\pi_{st} = \nu^\pi_{st} P^\pi$. 
We assume the following condition on the convergence to stationary distribution, which is commonly used in the MDP literature \cite[see][]{yu2009markov,even2009online,neu2014online}.

\begin{assumption}\label{assumption:mixing}
There exists a real number $\tau \geq 0$ such that for any policy $\pi \in \Pi$ and any pair of distributions $\nu,\nu' \in \Delta_S$, it holds that $\Vert \nu P^\pi - \nu' P^\pi\Vert_1 \leq e^{-\frac{1}{\tau}}\Vert \nu - \nu'\Vert_1$.
\end{assumption}

We refer to $\tau$ in Assumption~\ref{assumption:mixing} as the \emph{mixing time}, which measures the convergence speed to the stationary distribution. In particular, the assumption implies that  $\nu^\pi_{st}$ is unique for a given policy $\pi$.

We use $\mu(s,a)$ to denote the proportion of time that the MDP visits state-action pair $(s,a)$ in the long run. We call $\mu^\pi \in \mathbb{R}^{\vert S\vert \times \vert A\vert}$ the \emph{occupancy measure} of policy $\pi$. 
Let $\rho_t^\pi $ be the long-run average reward under policy $\pi$ when the reward function is fixed to be $r_t$ every period, i.e., $\rho^\pi_t \triangleq  \lim_{T\rightarrow \infty} \frac{1}{T} \sum_{i=1}^T\mathbb{E}[ r_t(s^\pi_i,a^\pi_i) ] $. We define $\rho_t \triangleq \rho_t^{\pi_t}$, where $\pi_t$ is the policy selected by the decision maker for time $t$.


\subsection{Linear Programming Formulation for the Average Reward MDP}\label{mdp_via_lp}
Given a reward function $r: S \times A \rightarrow [-1,1]$, suppose one wants to find a policy $\pi$ that maximizes the long-run average reward: $\rho^*=\sup_{\pi}\lim_{T\rightarrow \infty} \frac{1}{T}\sum_{t=1}^T r(s^\pi_t,a^\pi_t)$.
Under Assumption~\ref{assumption:mixing}, the Markov chain induced by any policy is ergodic and the long-run average reward is independent of the starting state \cite{bertsekas1995dynamic}. 
It is well known that 
the optimal policy can be obtained by solving the Bellman equation, which in turn can be written as a linear program (in the dual form): 
\begin{align}
\rho^* = \max_\mu & \sum_{s\in S} \sum_{a \in A} \mu(s,a)r(s,a) \label{eq:LP} \\
\text{s.t. } &  \sum_{s\in S} \sum_{a\in A} \mu(s,a) P(s' \vert s,a) = \sum_{a\in A} \mu (s',a) \quad \forall s' \in S \nonumber \\
&\sum_{s\in S} \sum_{a\in A} \mu(s,a) = 1,\quad \mu(s,a)\geq 0 \quad \forall s\in S,\,\forall a\in A. \nonumber 
\end{align}
Let $\mu^*$ be an optimal solution to the LP \eqref{eq:LP}. We can construct an optimal policy of the MDP by defining $ \pi^*(s,a) \triangleq \frac{\mu^*(s,a)}{\sum_{a\in A} \mu^*(s,a)}$ for all $s\in S$ such that $\sum_{a\in A} \mu^*(s,a)>0$;  for states where the denominator is zero, the policy may choose  arbitrary actions, since those states will not be visited in the stationary distribution.
 Let $\nu^*_{st}$ be the stationary distribution over states under this optimal policy.  

For simplicity, we will write the first constraint of LP \eqref{eq:LP}
in the matrix form as $\mu^\top (P-B)=0$, for appropriately chosen matrix $B$. We denote the feasible set of the above LP as $\Delta_M \triangleq \{\mu\in \mathbb{R}: \mu \geq 0, \mu^\top1=1, \mu^\top(P-B)=0 \}$. 
The following definition will be used in the analysis later.
\begin{definition}\label{def:delta_0}
Let $\delta_0 \geq 0$ be the largest real number such that for all $\delta \in[0,\delta_0]$, the set $\Delta_{M,\delta}\triangleq\{ \mu \in \mathbb{R}^{\vert S \vert \times \vert A \vert}: \mu \geq \delta, \mu^\top 1 = 1, \mu^\top(P-B)=0 \}$ is nonempty.
\end{definition}

\section{A Sublinear Regret Algorithm for Online MDP}\label{sec:mdp_rftl}

In this section, we present an algorithm for the Online MDP problem.

\begin{algorithm}[!htb]
\caption{(\textsc{MDP-RFTL})}
\label{alg:MDP-RFTL}
\begin{algorithmic}
  \STATE {\bfseries input:} parameter $\delta>0, \eta>0$,  regularization term $R(\mu) = \sum_{s\in S} \sum_{a\in A} \mu(s,a) \ln(\mu(s,a))$
  \STATE {\bfseries initialization:} choose any $\mu_1 \in \Delta_{M,\delta} \subset \mathbb{R}^{|S|\times|A|}$
  \FOR{$t=1,...T$}  
  \STATE observe current state $s_t$
  \IF{$\sum_{a\in A}\mu_t(s_t,a) > 0$} 
  	\STATE {choose action $a \in A$ with probability $\frac{\mu_t(s_t,a)}{\sum_{a}\mu_t(s_t,a)}$.}
  \ELSE 
  	\STATE{choose action $a\in A$ with probability $\frac{1}{|A|}$} 
  \ENDIF  
  \STATE observe reward function $r_t \in [-1,1]^{\vert S\vert \vert A \vert}$
  \STATE update $\mu_{t+1}\leftarrow \arg \max_{\mu \in \Delta_{M,\delta}} \sum_{i=1}^t \left[ \langle r_i , \mu \rangle - \frac{1}{\eta}R(\mu) \right]$
  \ENDFOR
\end{algorithmic}
\end{algorithm}

At the beginning of each round $t\in [T]$, the algorithm starts with an occupancy measure $\mu_t$. If the MDP is in state $s_t$, we play action $a\in A$ with probability $\frac{\mu_t(s_t,a)}{\sum_{a}\mu_t(s_t,a)}$. If the denominator is 0, the algorithm picks any action in $A$ with equal probability.
After observing reward function $r_t$ and collecting reward $r_t(s_t,a_t)$,  the algorithm changes the occupancy measure to $\mu_{t+1}$. 

The new occupancy measure is chosen according to the Regularized Follow the Leader (RFTL) algorithm
 \cite{shalev2007online,abernethy2009competing}. RFTL chooses the best occupancy measure for the cumulative reward  observed so far $\sum_{i=1}^t r_i$, plus a regularization term $R(\mu)$. The regularization term forces the algorithm  not to drastically change the occupancy measure from round to round.  In particular, we choose $R(\mu)$ to be the entropy function.

The complete algorithm is shown  in Algorithm~\ref{alg:MDP-RFTL}. 
The main result of this section is the following.

\begin{theorem}\label{theorem:mdp_rftl}
Suppose $\{r_t\}_{t=1}^T$ is an arbitrary sequence of rewards such that $\vert r_t(s,a)\vert \leq 1$ for all $s\in S$ and $a\in A$. For $T\geq \ln^2({1}/{\delta_0})$, the MDP-RFTL algorithm with parameters $\eta = \sqrt{\frac{T \ln(\vert S \vert \vert A \vert)}{\tau}}$, $\delta = e^{-{\sqrt{T}}/{\sqrt{\tau}}}$ guarantees
\begin{align*}
\text{MDP-Regret}(T) \leq O \left( \tau + 4 \sqrt{\tau T  (\ln\vert S \vert + \ln \vert A \vert)} \ln(T) \right).
\end{align*}
\end{theorem}

The regret bound in Theorem~\ref{theorem:mdp_rftl} is near optimal: a lower bound of $\Omega(\sqrt{T\ln|A|})$ exists for the problem of learning with expert advice \cite{freund1999adaptive,hazan2016introduction}, a special case of Online MDP where the state space is a singleton.
We note that the bound only depends \emph{logarithmically} on the size of the state space and action space.
The state-of-the-art regret bound for Online MDPs is that of \cite{even2009online}, which is $O(\tau + \tau^2 \sqrt{\ln(|A|)T})$. Compared to their result, our bound is better by a factor of $\tau^{3/2}$. However, our bound has depends on $\sqrt{\ln|S|+\ln|A|}$, whereas the bound in \cite{even2009online} depends on $\sqrt{\ln |A|}$.
Both algorithms require $poly(|S||A|)$ computation time, but are based on different ideas:
 The algorithm of \cite{even2009online} is based on expert algorithms and requires computing $Q$-functions at each time step, whereas our algorithm is based on RFTL. In the next section, we will show how to extend our algorithm to the case with large state space.
 

\subsection{Proof Idea for Theorem \ref{theorem:mdp_rftl}}

The key to analyze the algorithm is to decompose the regret with respect to policy $\pi \in \Pi$ as follows 
{
\medmuskip=1mu
\begin{align}\label{regret_decomposition}
R(T, \pi)  =  \left[\mathbb{E}[\sum_{t=1}^T r_t(s^\pi_t , a^\pi_t)] - \sum_{t=1}^T \rho^\pi_t\right] + \left[\sum_{t=1}^T \rho^\pi_t - \sum_{t=1}^T \rho_t \right] + \left[  \sum_{t=1}^T \rho_t -  \mathbb{E}[\sum_{t=1}^T r_t(s_t,a_t)] \right].
\end{align}
}
This decomposition was first used by \cite{even2009online}. We now give some intuition on why $R(T, \pi)$ should be sublinear. By the mixing condition in Assumption~\ref{assumption:mixing}, the state distribution $\nu^\pi_t$ at time $t$ under a policy $\pi$ differs from the stationary distribution $\nu^\pi_{st}$ by at most $O(\tau)$. This result can be used to bound the first term of \eqref{regret_decomposition}.

The second term of \eqref{regret_decomposition} can be related to the online convex optimization (OCO) problem through the linear programming formulation from Section~\ref{mdp_via_lp}. Notice that $ \rho^\pi_t = \sum_{s\in S} \sum_{a\in A} \mu^\pi(s,a) r(s,a) = \langle \mu^\pi, r\rangle $, and $ \rho_t = \sum_{s\in S} \sum_{a\in A} \mu^\pi_t(s,a) r(s,a) = \langle \mu^{\pi_t}, r\rangle$.
Therefore, we have that 
\begin{align}
\sum_{t=1}^T \rho^\pi_t - \sum_{t=1}^T \rho_t  = \sum_{t=1}^T\langle \mu^\pi, r_t\rangle - \sum_{t=1}^T \langle \mu^{\pi_t}, r_t \rangle,
\end{align}
which is exactly the regret quantity commonly studied in OCO. 
We are thus seeking an algorithm that can bound $\max_{\mu \in \Delta_M}  \sum_{t=1}^T\langle \mu^\pi, r_t\rangle - \sum_{t=1}^T \langle \mu^{\pi_t}, r_t \rangle$. In order to achieve logarithmic dependence on $|S|$ and $|A|$ in Theorem~\ref{theorem:mdp_rftl}, we apply the RFTL algorithm, regularized by the negative entropy function  $R(\mu)$. A technical challenge we faced in the analysis is that $R(\mu)$ is not Lipschitz continuous over $\Delta_M$, the feasible set of LP \eqref{eq:LP}. So we design the algorithm to play in a shrunk set $\Delta_{M,\delta}$ for some $\delta >0$ (see Definition~\ref{def:delta_0}), in which $R(\mu)$ is indeed Lipschitz continuous.

For the last term in \eqref{regret_decomposition}, note that it is similar to the first term, although more complicated: the policy $\pi$ is fixed in the first term, but the policy $\pi_t$ used by the algorithm is varying over time. To solve this challenge, the key idea is to show that the policies do not change too much from round to round, so that the third term grows sublinearly in $T$. To this end, we use the property of the RFTL algorithm with carefully chosen regularization parameter $\eta>0$.
The complete proof of Theorem \ref{theorem:mdp_rftl} can be found in Appendix \ref{sec:regret_analysis}.

\section{Online MDPs with Large State Space}\label{sec:large_mdps}
In the previous section, 
we designed an algorithm for Online MDP with sublinear regret.
However, the computational complexity of our algorithm
is $O(poly(\vert S \vert \vert A \vert))$  per round. In practice, 
MDPs often have extremely large state space $S$ due to the curse of dimenionality \cite{bertsekas1995dynamic}, so computing the exact solution becomes impractical.
In this section we propose an approximate algorithm that can handle large state spaces.

\subsection{Approximating Occupancy Measures and Regret Definition}

We consider an approximation scheme introduced in \cite{abbasi2014linear} for standard MDPs. The idea is to use $d$ feature vectors (with $d \ll \vert S \vert \vert A \vert$) to approximate occupancy measures $\mu \in \mathbb{R}^{|S|\times|A|}$.
Specifically, we approximate $\mu \approx \Phi \theta$ where $\Phi$ is a given matrix of dimension $\vert S \vert \vert A \vert \times d$, and $\theta \in \Theta \triangleq \{ \theta \in \mathbb{R}_+^d: \Vert \theta \Vert_1 \leq W\}$ for some positive constant $W$.
As we will restrict the occupancy measures chosen by our algorithm to satisfy $\mu = \Phi \theta$, the definition of MDP-regret \eqref{regret_def} is too strong as it compares against all stationary policies.
Instead, we restrict the benchmark to be the set of policies $\Pi^\Phi$ that can be represented by matrix $\Phi$, where
\begin{align*}
\Pi^\Phi \triangleq \{ \pi \in \Pi : \text{ there exists $\mu^\pi \in \Delta_M$ such that } \mu^\pi = \Phi \theta  \text{ for some } \theta \in \Theta \}.
\end{align*}
Our goal will now be to achieve sublinear $\Phi$-MDP-regret defined as
\begin{align}\label{eq:regret_def_phi}
\text{$\Phi$-MDP-Regret}(T) \triangleq \max_{\pi \in \Pi^\Phi} \mathbb{E}[\sum_{t=1}^T r_t(s^\pi_t , a^\pi_t)] - \mathbb{E}[\sum_{t=1}^T r_t(s_t,a_t)],
\end{align} 
where the expectation is taken with respect to random state transitions of the MDP and randomization used in the algorithm. Additionally, we want to make the computational complexity \emph{independent} of $\vert S \vert$ and $\vert A \vert$.

{\bf Choice of Matrix $\Phi$ and Computation Efficiency.}
The columns of matrix $\Phi \in \mathbb{R}^{\vert S \vert \vert A \vert \times d}$ represent probability distributions over state-action pairs.  
The choice of $\Phi$ is problem dependent, and a detailed discussion is beyond the scope of this paper.
\cite{abbasi2014linear} shows that for many applications such as the game of Tetris and queuing networks, $\Phi$ can be naturally chosen as a sparse matrix, which allows
constant time access to entries of $\Phi$ and efficient dot product operations. We will assume such constant time access throughout our analysis.
We refer readers to \cite{abbasi2014linear} for further details.

\subsection{The Approximate Algorithm}

The algorithm we propose is built on \textsc{MDP-RFTL}, but is significantly modified in several aspects. In this section, we start with key ideas on how and why we need to modify the previous algorithm, and then formally present the new algorithm. 

To aid our analysis, we make the following definition.
\begin{definition}\label{assumption:large-MDP}
Let $\tilde{\delta}_0 \geq 0$ be the largest real number such that for all $\delta \in[0,\tilde{\delta}_0]$ the set $\Delta^\Phi_{M,\delta}\triangleq\{ \mu \in \mathbb{R}^{\vert S \vert \vert A \vert}: \text{ there exists $\theta \in \Theta$ such that } \mu = \Phi \theta, \mu \geq \delta, \mu^\top 1 = 1, \mu^\top(P-B)=0 \}$ is nonempty. We also write $\Delta_M^\Phi \triangleq \Delta^\Phi_{M,0}$.
\end{definition}

As a first attempt, one could replace the shrunk set of occupancy measures $\Delta_{M,\delta}$ in Algorithm~\ref{alg:MDP-RFTL} with $\Delta^\Phi_{M,\delta}$ defined above. We then use occupancy measures $\mu^{\Phi \theta^*_{t+1}} \triangleq \Phi \theta^*_{t+1}$ given by the RFTL algorithm, i.e.,
$ \theta^*_{t+1} \leftarrow \arg \max_{\theta \in \Delta_{M,\delta}^\Phi} \sum_{i=1}^t \left[ \langle r_i, \mu \rangle - ({1}/{\eta}) R(\mu) \right]$. 
The same proof of Theorem~\ref{theorem:mdp_rftl} would apply and guarantee a sublinear $\Phi$-MDP-Regret.
Unfortunately, replacing $\Delta_{M,\delta}$ with $\Delta^\Phi_{M,\delta}$ does not reduce the time complexity of computing the iterates $\{\mu^{\Phi \theta^*_{t}}\}_{t=1}^T$, which is still $poly(|S||A|)$.

To tackle this challenge, we will not apply the RFTL algorithm exactly, but will instead obtain an approximate solution in $poly(d)$ time. We relax the constraints $\mu \geq \delta$ and $\mu^\top(P-B) = 0$ that define the set $\Delta_{M,\delta}^\Phi$, and add the following penalty term to the objective function: 
\begin{equation}\label{eq:penalty}
    V(\theta) \triangleq -H_t \Vert (\Phi \theta )^\top (P-B) \Vert_1 - H_t \Vert \min\{\delta, \Phi \theta \} \Vert_1.
\end{equation}
Here, $\{H_t\}_{t=1}^T$ is a sequence of tuning parameters that will be specified in Theorem~\ref{thm:large_mdp_regret}.
Let $ \Theta^\Phi \triangleq \{\theta \in \Theta\ , \mathbf{1}^\top (\Phi \theta)  = 1 \}$.
Thus, the original RFTL step in Algorithm~\ref{alg:MDP-RFTL} now becomes
\begin{equation}\label{eq:c_t_eta}
\max_{\theta \in \Theta^\Phi} \sum_{i=1}^tc^{t,\eta}(\theta), \quad \text{where }
c^{t,\eta}(\theta) \triangleq \sum_{i=1}^t \left[\langle r_i , \Phi \theta \rangle -  \frac{1}{\eta} R^\delta (\Phi \theta)\right] + V(\theta).
\end{equation}
In the above function, we use a modified entropy function $R^\delta(\cdot)$ as the regularization term, because the standard entropy function has an infinite gradient at the origin.
More specifically, let 
$R_{(s,a)}(\mu)\triangleq \mu(s,a) \ln(\mu(s,a))$ be the entropy function. We define $ R^\delta (\mu) = \sum_{(s,a)} R^\delta_{(s,a)}(\mu (s,a))$, where
\begin{align}\label{eq:def_R_delta}
R^\delta_{(s,a)} \triangleq
\begin{cases}
R_{(s,a)}(\mu) \quad &\text{if } \mu(s,a) \geq \delta \\
R_{(s,a)}(\delta) +  \frac{d}{d \mu(s,a)} R_{(s,a)}(\delta) (\mu(s,a)-\delta) \quad &\text{otherwise}.
\end{cases}
\end{align}


Since computing an exact gradient for function $c^{t,\eta}(\cdot)$ would take $O(\vert S \vert \vert A \vert)$  time, we solve problem \eqref{eq:c_t_eta} by stochastic gradient ascent. The following lemma shows how to efficiently generate stochastic subgradients for function $c^{t,\eta}$ via sampling.
\begin{lemma}\label{stochastic_gradient_of_c}
Let $q_1$ be any probability distribution over state-action pairs, and $q_2$ be any probability distribution over all states. Sample a pair $(s',a')\sim q_1$ and $s'' \sim q_2$. The quantity 
\begin{align*}
g_{s',a',s''}(\theta) &= \Phi^\top r_{t} + \frac{H_t}{q_1(s',a')} \Phi_{(s',a'),:} \mathbb{I}\{\Phi_{(s',a'),:} \leq \delta \}\\
&\quad - \frac{H_t}{q_2(s'')} [(P-B)^\top \Phi]_{s'',:} sign([(P-B)^\top \Phi ]_{s'',:} \theta) - \frac{t}{\eta q_1(s',a')} \nabla_\theta R^\delta_{(s',a')}(\Phi \theta)
\end{align*}
satisfies $\mathbb{E}_{(s',a')\sim q_1, s'' \sim q_2}[g_{s',a',s''}(\theta) \vert \theta] = \nabla_\theta c^{\eta, t}(\theta)$ for any $\theta \in \Theta$. Morever, we have $ \Vert g(\theta) \Vert_2 \leq t \sqrt{d} + H_t (C_1 + C_2)+ \frac{t}{\eta}(1 + \ln (Wd) + \vert \ln(\delta) \vert) C_1,
$ w.p.1, where
\begin{equation}\label{def_of_Cs}
C_1 = \max_{(s,a)\in S \times A} \frac{\Vert \Phi_{(s,a),:}\Vert_2}{q_1(s,a)}, \quad C_2 = \max_{s\in S} \frac{\Vert (P-B)^\top_{:,s} \Phi \Vert_2}{q_2(s)}.
\end{equation}
\end{lemma}

Putting everything together, we present the complete approximate algorithm for large state online MDPs in Algorithm~\ref{alg:SGD-MDP-RFTL}.  
The algorithm uses Projected Stochastic Gradient Ascent (Algorithm~\ref{alg:P-SGA}) as a subroutine, which uses the sampling method in Lemma~\ref{stochastic_gradient_of_c} to generate stochastic sub-gradients.

\begin{algorithm}[tbh]
\caption{(\textsc{Large-MDP-RFTL})}
\label{alg:SGD-MDP-RFTL}
\begin{algorithmic}
  \STATE {\bfseries input:} matrix $\Phi$, parameters: $\eta, \delta >0$, convex function $R^\delta(\mu)$, SGA step-size schedule $\{w_t\}_{t=0}^T$, penalty term parameters $\{H_t\}_{t=1}^T$
  
  \STATE {\bfseries initialize:} $\tilde{\theta}_1\leftarrow $ \textsc{PSGA}($- R^{\delta}(\Phi \theta) + V(\theta), \Theta^\Phi, w_0, K_0  $)  
  \FOR{$t=1,...,T$}
  \STATE observe current state $s_t$; play action $a$ with distribution $\frac{[\Phi \tilde{\theta}_t]_+ (s_t,a)}{\sum_{a\in A} [\Phi \tilde{\theta}_t]_+ (s_t,a)}$ 
  \STATE observe $r_t \in [-1,1]^{\vert S\vert \vert A \vert}$
  \STATE $\tilde{\theta}_{t+1}\leftarrow $ \textsc{PSGA}($\sum_{i=1}^t [ \langle r_i, \Phi \theta \rangle -\frac{1}{\eta} R^\delta(\Phi \theta) ] + V(\theta), \Theta^\Phi, w_t, K_t  $)
  \ENDFOR
\end{algorithmic}
\end{algorithm}


\begin{algorithm}[tbh]
\caption{Projected Stochastic Gradient Ascent: \textsc{PSGA}$(f, X, w, K)$}
\label{alg:P-SGA}
\begin{algorithmic}
  \STATE {\bfseries input:} concave objective function $f$, feasible set $X$, stepsize $w$, $x_1 \in X$
    \FOR{$k=1,...K$}
  \STATE compute a stochastic subgradient $g_k$ such that $\mathbb{E}[g_k] = \nabla f(x_k)$ using Lemma \ref{stochastic_gradient_of_c}
  \STATE set $x_{k+1} \leftarrow P_X(x_k + w g(x_k))$ 
  \ENDFOR
  \STATE {\bfseries output:} $\frac{1}{K} \sum_{k=1}^K x_k$
\end{algorithmic}
\end{algorithm}

\subsection{Analysis of the Approximate Algorithm}


We establish a regret bound for the \textsc{Large-MDP-RFTL} algorithm as follows.

\begin{theorem}\label{thm:large_mdp_regret}
Suppose $\{r_t\}_{t=1}^T$ is an arbitrary sequence of rewards such that $\vert r_t(s,a)\vert \leq 1$ for all $s\in S$ and $a\in A$. 
For $T \geq \ln^2(\frac{1}{\delta_0})$, \textsc{Large-MDP-RFTL} with parameters $\eta = \sqrt{\frac{T}{\tau}}, \delta = e^{-\sqrt{T}}$, $K(t) = \left[ W t^2 \sqrt{d} \tau^4 (C_1 + C_2) T\ln(W T) \right]^2$, $H_t = t \tau^2 T^{3/4}$, $w_t=\frac{W}{\sqrt{K(t)} (t \sqrt{d} + H_t (C_1 + C_2) + \frac{t}{\eta} C_1)} $ guarantees that 
\begin{align*}
\text{$\Phi$-MDP-Regret}(T) \leq O( c_{S,A} \ln(\vert S \vert  \vert A\vert) \sqrt{\tau T} \ln(T) ).
\end{align*}
Here $c_{S,A}$ is a problem dependent constant. The constants $C_1, C_2$ are defined in Lemma~\ref{stochastic_gradient_of_c}.
\end{theorem}

A salient feature of the \textsc{Large-MDP-RFTL} algorithm is that its computational complexity in each period is independent of the size of state space $|S|$ or the size of action space $|A|$, and thus is amenable to large scale MDPs. In particular, in Theorem~\ref{thm:large_mdp_regret}, the number of SGA iterations, $K(t)$, is $O(d)$ and independent of $|S|$ and $|A|$.

Compared to Theorem~\ref{theorem:mdp_rftl}, we achieve a regret with similar dependence on the number of periods $T$ and the mixing time $\tau$. The regret bound also depends on $\ln(|S|)$ and $\ln(|A|)$, with an additional constant term $c_{S,A}$. The constant comes from a projection problem (see details in Appendix~\ref{sec:regret_analysis_large_mdp}) and may grow with $|S|$ and $|A|$ in general. But for some classes of MDP problem, $c_{S,A}$ is bounded by an absolute constant: an example is the Markovian Multi-armed Bandit problem \cite{whittle1980multi}.

{\bf Proof Idea for Theorem~\ref{thm:large_mdp_regret}.} Consider the \textsc{MDP-RFTL} iterates ,\{$\theta^*_t\}_{t=1}^T$, and the occupancy measures $\{\mu^{\Phi \theta^*_t}\}_{t=1}^T$ induced by following policies $\{\Phi \theta^*_t\}_{t=1}^T$. Since $\theta^*_t \in \Delta_{M,\delta}^\Phi$ it holds that $\mu^{\Phi \theta^*_t} = \Phi \theta^*_t$ for all $t$. Thus, following the proof of Theorem \ref{theorem:mdp_rftl}, we can obtain  the same $\Phi$-MDP-Regret bound in Theorem \ref{theorem:mdp_rftl} if we follow policies $\{\Phi \theta^*_t\}_{t=1}^T$. However, computing $\theta^*_t$ takes $O(poly(\vert S \vert \vert A \vert))$ time. 

The crux the proof of Theorem \ref{thm:large_mdp_regret} is to show that the $\{\Phi \tilde{\theta}_{t}\}_{t=1}^T$ iterates in Algorithm~\ref{alg:SGD-MDP-RFTL} induce occupancy measures $\{\mu^{\Phi \tilde{\theta}_{t}}\}_{t=1}^T$ that are close to $\{\mu^{\Phi \theta^*_t}\}_{t=1}^T$. Since the algorithm has relaxed constraints of $\Delta_{M,\delta}^\Phi$, in general we have $\tilde{\theta}_{t} \notin \Delta_{M,\delta}^\Phi$ and thus $\mu^{\Phi \tilde{\theta}_{t}} \neq \Phi \tilde{\theta}_{t}$. So we need to show that the distance between $\mu^{\Phi \theta^*_{t+1}}$, and $\mu^{\Phi \tilde{\theta}_{t+1}}$ is small. Using triangle inequality we have 
\begin{align*}
\Vert \mu^{\Phi \theta_t^*} - \mu^{\Phi \tilde{\theta}_t}  \Vert_1 &\leq \Vert \mu^{\Phi \theta_t^*} - P_{\Delta_{M,\delta}^\Phi}(\Phi \tilde{\theta}_t)  \Vert_1  +  \Vert P_{\Delta_{M,\delta}^\Phi}(\Phi \tilde{\theta}_t) - \Phi \tilde{\theta}_t  \Vert_1 + \Vert \Phi \tilde{\theta}_t - \mu^{\Phi \tilde{\theta}_t} \Vert_1, 
\end{align*}
where $P_{\Delta_{M,\delta}^\Phi}(\cdot)$ denotes the Euclidean projection onto $\Delta_{M,\delta}^\Phi$. We then proceed to bound each term individually. We defer the details to Appendix \ref{sec:regret_analysis_large_mdp} as bounding each term requires lengthy proofs. 

\section{Conclusion}

We consider Markov Decision Processes (MDPs) where the transition probabilities are known but
the rewards are unknown and may change in an adversarial manner.  We provide an online algorithm, which applies Regularized Follow the Leader (RFTL) to the linear programming formulation of the average reward MDP. The algorithm achieves a
regret bound of $O( \sqrt{\tau (\ln|S|+\ln|A|)T}\ln(T))$, where $S$ is the state space, $A$ is the action space, $\tau$ is the mixing time of the MDP, and $T$ is the number of periods. The algorithm's computational complexity  is polynomial in $|S|$ and $|A|$ per period. 

We then consider a setting often encountered in practice, where the state space of the MDP is too large to allow for exact solutions. We approximate the state-action occupancy measures with a linear architecture of dimension $d\ll|S||A|$. We then propose an approximate algorithm which relaxes the constraints in the RFTL algorithm, and solve the relaxed problem using stochastic gradient descent method. 
A salient feature of our algorithm is that its computation time complexity is independent of the size of state space $|S|$ and the size of action space $|A|$. 
We prove a regret bound of $O( c_{S,A} \ln(\vert S \vert  \vert A\vert) \sqrt{\tau T} \ln(T))$ compared to the best static policy approximated by the linear architecture, where $c_{S,A}$ is a problem dependent constant.
To the best of our knowledge, this is the first $\tilde{O}(\sqrt{T})$ regret bound for large scale MDPs with changing rewards.

\bibliography{mybib}
\bibliographystyle{abbrv} 
\newpage

\begin{appendix}
\section{Regret Analysis of MDP-RFTL: Proof of Theorem~\ref{theorem:mdp_rftl}}\label{sec:regret_analysis}
To bound the regret incurred by \textsc{MDP-RFTL}, we bound each term in Eq~\eqref{regret_decomposition}. We start with the first term. We use the following lemma, which was first stated in \cite{even2009online} and was also used by \cite{neu2014online}.
\begin{lemma}\label{lemma:first_term}
For any $T\geq 1$ and any policy $\pi$ it holds that
\begin{align*}
\mathbb{E}[\sum_{t=1}^T r_t(s^\pi_t , a^\pi_t)] - \sum_{t=1}^T \rho^\pi_t \leq 2\tau + 2.
\end{align*}
\end{lemma}
\begin{proof} [Proof of Lemma \ref{lemma:first_term} ]
Recall that $|r_t(s,a)| \leq 1$, so we have $|\sum_{a\in A} \pi(s,a) r_t(s,a) |\leq 1$ by Cauchy-Schwarz inequality, since $\pi(s, \cdot)$ defines a probability distribution over actions. Also, recall that $\nu^\pi_t$ is the stationary distribution over states by following policy $\pi$ and $\nu^\pi_{t+1} = \nu^\pi_{t} P^\pi $ for all $t\in [T]$. We have
\begin{align*}
\mathbb{E}[\sum_{t=1}^T r_t(s^\pi_t , a^\pi_t)] &= \sum_{t=1}^T \sum_{s\in S}(\nu^\pi_t(s) - \nu^\pi_{st}(s)) \sum_{a\in A} \pi(s,a) r_t(s,a)\\
&\leq \sum_{t=1}^T \sum_{s\in S}\nu^\pi_t(s) - \nu^\pi_{st}(s)\\
&\leq \sum_{t=1}^T \Vert \nu^\pi_t(s) - \nu^\pi_{st}(s)\Vert_1.
\end{align*}
Now, notice that 
\begin{align*}
\Vert \nu^\pi_t(s) - \nu^\pi_{st}(s)\Vert_1 & = \Vert \nu^\pi_{t-1} P^\pi - \nu_{st} P^\pi \Vert_1\\
& \leq e^{-\frac{1}{\tau}}  \Vert \nu^\pi_{t-1} - \nu_{st}  \Vert_1 \quad \text{by Assumption 1}\\
&\leq e^{-\frac{t}{\tau}} \Vert \nu^\pi_1 - \nu^\pi_{st}\Vert_1 \quad \text{by repeating the argument $t-1$ more times}\\
&\leq 2e^{-\frac{t}{\tau}}.
\end{align*}
Finally, we have that 
\begin{align*}
\sum_{t=1}^T \Vert \nu^\pi_t(s) - \nu^\pi_{st}(s)\Vert_1 & \leq 2 \sum_{t=1}^T e^{-\frac{t}{\tau}}\\
& \leq 2 (1+ \int_{0}^{\infty} e^{-\frac{t}{\tau}})dt\\
&= 2\tau + 2,
\end{align*}
which concludes the proof.
\end{proof}

We now bound the third term in (\ref{regret_decomposition}). We use the following lemma, which bounds the difference of two stationary distributions by the difference of the corresponding occupancy measures.
\begin{lemma}\label{lemma:nu_mu}
Let $\nu^1_{st}$ and $\nu^2_{st}$ be two arbitrary stationary distributions over $S$. Let $\mu^1$ and $\mu^2$ be the corresponding occupancy mesures. It holds that 
\begin{align*}
\Vert \nu^1_{st} - \nu^2_{st} \Vert_1 \leq \Vert \mu^1 - \mu^2 \Vert_1.
\end{align*}
\end{lemma}

\begin{proof} [Proof of Lemma \ref{lemma:nu_mu}]
\begin{align*}
\Vert \nu^1_{st} - \nu^2_{st} \Vert_1 &= \sum_{s\in S} \vert \nu^1_{st}(s) - \nu^2_{st}(s)\vert \\
& = \sum_{s\in S} \vert \sum_{a\in A} \mu^1(s,a) - \mu^2(s,a)\vert \\
& \leq \sum_{s\in S}  \sum_{a\in A} \vert \mu^1(s,a) - \mu^2(s,a)\vert \\
& = \Vert \mu^1 - \mu^2\Vert_1.
\end{align*}
\end{proof}

We are ready to bound the third term in (\ref{regret_decomposition}).
\begin{lemma}\label{lemma:third_term}
Let $\{s_t,a_t\}_{t=1}^T$ be the random sequence of state-action pairs generated by the policies induced by occupancy measures $\{\mu^{\pi_t}\}_{t=1}^T$. It holds that 
\begin{align*}
\sum_{t=1}^T \rho_t - \mathbb{E} \left[  \sum_{t=1}^T r_t(s_t,a_t) \right]  \leq \sum_{t=1}^T 2 e^{-\frac{t-1}{\tau}} + \sum_{t=1}^T \sum_{\theta=0}^{t-1} e^{-\frac{\theta}{\tau}} \Vert \mu^{\pi_{t-\theta}} - \mu^{\pi_{t-(\theta+1)}} \Vert_1.
\end{align*}
\end{lemma}

\begin{proof}[Proof of Lemma \ref{lemma:third_term}]
By the definition of $\rho_t$, we have
\begin{align*} 
\sum_{t=1}^T \rho_t - \mathbb{E} \left[  \sum_{t=1}^T r_t(s_t,a_t) \right] &= \sum_{t=1}^T \sum_{s\in S} (\nu^{\pi_t}_{st}(s) - \nu^t(s)) \sum_{a\in A} \pi^t(s,a)r_t(s,a)\\
&\leq \sum_{t=1}^T \Vert \nu_{st}^{\pi_t} - \nu^t\Vert_1.
\end{align*}

Now, recall that $\nu_t = \nu_1 P^{\pi_1} P^{\pi_2}...P^{\pi_{t-1}}$. We now bound $\Vert \nu^{\pi_t}_{st} - \nu^{t}\Vert_1$ for all $t \in [T]$ as follows:
\begin{align*}
\Vert \nu^{t}  - \nu^{\pi_t}_{st} \Vert_1 &\leq \Vert \nu^t - \nu_{st}^{\pi_{t-1}}\Vert_1 + \Vert  \nu_{st}^{\pi_{t-1}} -  \nu^{\pi_t}_{st} \Vert_1\\
&\leq \Vert \nu^t - \nu_{st}^{\pi_{t-1}}\Vert_1 + \Vert \mu^{\pi_{t-1}} - \mu^{\pi_t}\Vert_1 \quad \text{by Lemma \ref{lemma:nu_mu}}\\
&= \Vert \nu^{t-1}P^{\pi_{t-1}} - \nu_{st}^{\pi_{t-1}} P^{\pi_{t-1}}\Vert_1 + \Vert \mu^{\pi_{t-1}} - \mu^{\pi_t}\Vert_1 \\
&\leq e^{-\frac{1}{\tau}} \Vert \nu^{t-1} - \nu_{st}^{\pi_{t-1}} \Vert_1+ \Vert \mu^{\pi_{t-1}} - \mu^{\pi_t}\Vert_1 \quad \text{by Assumption 1}\\
&\leq e^{-\frac{1}{\tau}}  ( e^{-\frac{1}{\tau}}  \Vert \nu^{t-2} - \nu_{st}^{\pi_{t-2}} \Vert_1 + \Vert \mu^{\pi_{t-2}} - \mu^{\pi_{t-1}} \Vert_1 )+ \Vert \mu^{\pi_{t-1}} - \mu^{\pi_t}\Vert_1 \\
&\leq e^{-\frac{t-1}{\tau}} \Vert \nu^1 - \nu_{st}^{\pi_1} \Vert_1 + \sum_{\theta=0}^{t-1} e^{-\frac{\theta}{\tau}}\Vert \mu^{\pi_{t-\theta}}-\mu^{\pi_{t-(\theta+1)}}\Vert_1,
\end{align*}
which yields the desired claim. 
\end{proof}

Combining Lemma~\ref{lemma:first_term}, Lemma~\ref{lemma:third_term} and Eq \eqref{regret_decomposition},
we have arrived at the following bound on the regret:
{
\medmuskip=-2mu
\begin{align*}
R(T,\pi) \leq (2\tau + 2) + \left[ \sum_{t=1}^T\langle \mu^\pi, r_t\rangle - \sum_{t=1}^T \langle \mu^{\pi_t}, r_t \rangle \right] + \left[ \sum_{t=1}^T 2 e^{-\frac{t-1}{\tau}} + \sum_{t=1}^T \sum_{\theta=0}^{t-1} e^{-\frac{\theta}{\tau}} \Vert \mu^{\pi_{t-\theta}} - \mu^{\pi_{t-(\theta+1)}} \Vert_1 \right].
\end{align*}
}
To complete the proof, we want to bound the second and the third terms. For the second term $\max_{\mu \in \Delta_M}  \sum_{t=1}^T\langle \mu^\pi, r_t\rangle - \sum_{t=1}^T \langle \mu^{\pi_t}, r_t \rangle$, since the reward functions are linear in $\mu$ and the set $\Delta_M$ is convex, any algorithm for Online Linear Optimization, e.g., Online Gradient Ascent \cite{zinkevich2003online}, ensures a regret bound that is sublinear $T$. However, this would yield an MDP-regret rate that depends linearly on $\vert S \vert \times \vert A \vert$. 

Instead, by noticing that the feasible set of the LP, $\Delta_M$, is a subset of the probability simplex $\Delta^{\vert S \vert \vert A \vert}$, we use RFTL and regularize using the negative entropy function. This will give us a rate that scales as $\ln(\vert S \vert \vert A \vert)$, which is much more desirable than $O(\vert S \vert \vert A \vert)$. Notice that the algorithm does not work with the set $\Delta_M$ directly but with $\Delta_{M,\delta}$ instead, this is because the negative entropy is not Lipschitz over $\Delta_M$. Working over $\Delta_{M,\delta}$ is the key to being able to bound the third term in the regret decomposition. Formally, we have the following result.

\begin{lemma} \label{lemma:btl_ftl}
Let $\{\mu_t\}_{t=1}^T$ be the iterates of MDP-RFTL, then it holds that
\begin{align*}
\max_{\mu \in \Delta_{M,\delta}} \sum_{t=1}^T \langle r_t,\mu \rangle \leq \sum_{t=1}^T \langle r_t, \mu^{\pi_{t+1}} \rangle + \frac{T}{\eta} \max_{\mu_1,\mu_2\in \Delta_{M,\delta}} \left[ R(\mu_1) - R(\mu_2) \right].
\end{align*}
\end{lemma}

\begin{proof}[Proof of Lemma \ref{lemma:btl_ftl}]
Define $f_t \triangleq \langle \mu, r_t \rangle$ and $f_t^R \triangleq f_t(\mu) - \frac{1}{\eta}R(\mu)$ for all $t=1,..,T$. We first prove by induction that 
\begin{align*}
\max_{\mu\in \Delta_{M,\delta}} \sum_{t=1}^T f_t^R(\mu) \leq \sum_{t=1}^T f_t^R(\mu^{\pi_{t+1}}).
\end{align*}
The base case $T=1$ is trivial by the definition of $\mu^{\pi_{2}}$. Suppose the claim holds for $T-1$. 
For all $\mu \in \Delta_{M,\delta}$, we have that
\begin{align*}
\sum_{t=1}^T f_t^R(\mu) & \leq \sum_{t=1}^T f_t^R(\mu^{\pi_{T+1}})\\
&\leq \max_{\mu \in \Delta_{M,\delta}} \sum_{t=1}^{T-1} f_t^R(\mu) +  f_T^R(\mu^{\pi_{T+1}})\\
&\leq \sum_{t=1}^{T-1} f_t^R(\mu^{\pi_{t+1}}) +  f_T^R(\mu^{\pi_{T+1}}) \quad \text{by induction hyposthesis}\\
& = \sum_{t=1}^{T} f_t^R(\mu^{\pi_{t+1}}).
\end{align*}
The lemma follows by plugging back in the definition of $f_t^R$ and rearranging terms.
\end{proof}

\begin{lemma}\label{lemma:close_iterates}
Let $\{\mu_t\}_{t=1}^T$ be the iterates of MDP-RFTL, it holds that
\begin{align*}
\Vert \mu^{\pi_t} - \mu^{\pi_{t+1}} \Vert_1 \leq \frac{2\eta}{t} \left( 1 + \frac{1}{\eta} G_R\right).
\end{align*}
\end{lemma}
\begin{proof}[Proof of Lemma \ref{lemma:close_iterates}]
Let $J(\mu) = \sum_{\theta=1}^{t} \left[ \langle \mu, r_t \rangle - \frac{1}{\eta} R(\mu) \right]$. Since $R$ is the negative entropy we know it is $1$- strongly convex with respect to norm $\Vert \cdot \Vert_1$, thus $J$ is $\frac{t}{\eta}$-strongly concave. By strong concavity we have
\begin{align*}
\frac{t}{2\eta} \Vert \mu^{\pi_{t+1}} - \mu^{\pi_{t} }\Vert_1^2 \leq J(\mu^{\pi_{t+1}}) - J (\mu^{\pi_{t}}) + \langle \nabla_\mu J(\mu^{\pi_{t+1}}), \mu^{\pi_t} - \mu^{\pi_{t+1}}\rangle. 
\end{align*}
Since $\mu^{\pi_{t+1}}$ is the optimizer of $J$ the optimality condition states that $\langle \nabla_\mu J(\mu^{\pi_{t+1}}), \mu^{\pi_t} - \mu^{\pi_{t+1}}\rangle \leq 0$. Plugging back in the definition of $J$ we have that
\begin{align*}
& \frac{t}{2\eta} \Vert \mu^{\pi_{t+1}} - \mu^{\pi_{t} }\Vert_1^2 \\
&\leq \sum_{\theta=1}^t \left[ \langle r_\theta, \mu^{\pi_{t+1}}\rangle - \frac{1}{\eta} R(\mu^{\pi_{t+1}}) \right] - \sum_{\theta=1}^t \left[ \langle r_\theta, \mu^{\pi_{t}}\rangle - \frac{1}{\eta} R(\mu^{\pi_{t}}) \right] \\
& = \sum_{\theta=1}^{t-1} \left[ \langle r_\theta, \mu^{\pi_{t+1}}\rangle - \frac{1}{\eta} R(\mu^{\pi_{t+1}}) \right] - \sum_{\theta=1}^{t-1} \left[ \langle r_\theta, \mu^{\pi_{t}}\rangle - \frac{1}{\eta} R(\mu^{\pi_{t}}) \right] \\
& \qquad + \langle r_t, \mu^{\pi_{t+1}}\rangle - \frac{1}{\eta} R(\mu^{\pi_{t+1}}) - \langle r_\theta, \mu^{\pi_{t}}\rangle + \frac{1}{\eta} R(\mu^{\pi_{t}}) \\
&\leq \langle r_t, \mu^{\pi_{t+1}}\rangle - \frac{1}{\eta} R(\mu^{\pi_{t+1}}) - \langle r_\theta, \mu^{\pi_{t}}\rangle + \frac{1}{\eta} R(\mu^{\pi_{t}}) \quad \text{ by definition of $\mu^{\pi_{t}}$}\\
&\leq \Vert r_t \Vert_\infty \Vert \mu^{\pi_{t}} - \mu^{\pi_{t+1}} \Vert_1 + \frac{1}{\eta} R(\mu^{\pi_{t}}) - \frac{1}{\eta} R(\mu^{\pi_{t+1}})\quad \text{ by Cauchy-Schwarz inequality}\\
&\leq \Vert \mu^{\pi_{t}} - \mu^{\pi_{t+1}} \Vert_1 + \frac{G_R}{\eta} \Vert \mu^{\pi_{t}} - \mu^{\pi_{t+1}} \Vert_1 \quad \text{Since $R$ is $G_R$- Lipschitz}.
\end{align*}
By rearranging terms, we get 
\begin{align*}
\Vert  \mu^{\pi_{t}} - \mu^{\pi_{t+1}} \Vert_1 \leq \frac{2\eta}{t} \left( 1 + \frac{1}{\eta} G_R \right).
\end{align*}
\end{proof}

Notice that by Lemma \ref{lemma:close_iterates} we will need the regularizer $R$ to be Lipschitz continuous with respect to norm $\Vert \cdot \Vert_1$. Unfortunately, the negative entropy function is not Lipschitz continuous over $\Delta_M$, so we will force the algorithm to play in a shrunk set $\Delta_{M,\delta}$. 

\begin{lemma}\label{lemma:entropy_lipschitz}
Let $\Delta_\delta \triangleq \{x\in \mathbb{R}^d : \Vert x \Vert_1 = 1, x_i \geq \delta \; \forall i=1,...,d\}$. The function $R(x)\triangleq \sum_{i=1}^d x_i \ln(x_i)$ is $G_R$-Lipschitz continuous with respect to $\Vert \cdot \Vert_1$ over $\Delta_\delta$ with $G_R = \max\{|\ln(\delta)|,1\}$.
\end{lemma}
\begin{proof}[Proof of Lemma \ref{lemma:entropy_lipschitz}]
We want to find $G_R>0$ such that $\Vert \nabla R(x)\Vert_\infty \leq G_R$ for all $x\in \Delta_\delta$. Notice that $[\nabla R(x)]_i = 1 + \ln(x_i)$ for $i=1,...d$. Moreover, since for every $i=1,...,d$ we have $\delta \leq x_i\leq 1$ the following sequence of inequalities hold: $\ln(\delta)\leq 1 +\ln(\delta) \leq 1+ \ln(x_i) \leq 1$. It follows that $G_R = \max\{|\ln(\delta)|,1\}$.
\end{proof}

The next lemma quantifies the loss in the regret due to playing in the shrunk set.
\begin{lemma}\label{lemma:shrinked_loss}
It holds that
\begin{align*}
\max_{\mu \in \Delta_{M}} \sum_{t=1}^T \langle r_t, \mu \rangle \leq \max_{\mu \in \Delta_{M,\delta}} \sum_{t=1}^T \langle r_t, \mu \rangle + 2 \delta T \left(\vert S\vert \vert A\vert - 1\right).
\end{align*}
\end{lemma}

\begin{proof}[Proof of Lemma \ref{lemma:shrinked_loss}]

Given $z^* \in \Delta \subset \mathbb{R}^d$, define $z^*_p \triangleq \arg \min_{z\in \Delta_\delta} \Vert z - z^*\Vert_1$, with $\delta \leq \frac{1}{d}$. 
It holds that $\Vert z^*_p - z^*\Vert_1 \leq 2\delta(d-1)$.
To see why the previous is true, choose $z^*=[1;0;0;...;0;0]$. It is easily verified that $z^*_p = [1-\delta (d-1); \delta; \delta; ...; \delta, \delta]$ and $\Vert z^* - z^*_p \Vert_1 = 2\delta(d-1)$. 
Because of the previous argument, if $\mu^*\in \arg \max _{\mu \in \Delta_{M}} \sum_{t=1}^T \langle r_t, \mu \rangle$ and $\mu^*_p$ is its $\Vert \cdot \Vert_1$ projection onto the the set $\Delta_{M,\delta}$ then $\Vert \mu^*- \mu^*_p\Vert_1 \leq 2 \delta (\vert S\vert \vert A\vert -1)$. The claim then follows since each function $\langle r_t, \mu \rangle$ is 1-Lipschitz continuous with respect to $\Vert \cdot \Vert_1$.

\end{proof}

Given that we know the iterates of MDP-RFTL are close by Lemma \ref{lemma:close_iterates}, we can bound the last term in our regret bound
\begin{lemma}\label{lemma:bound_third_term} It holds that
{
\medmuskip=0mu
\begin{align*}
 \sum_{t=1}^T 2 e^{-\frac{t-1}{\tau}} + \sum_{t=1}^T \sum_{\theta=0}^{t-1} e^{-\frac{\theta}{\tau}} \Vert \mu^{\pi_{t-\theta}} - \mu^{\pi_{t-(\theta+1)}} \Vert_1 \leq 2 (1 + \tau) + 2\eta \left( 1 + \frac{1}{\eta} G_R\right)(1+\ln(T)) (1 + \tau).
\end{align*}
}
\end{lemma}

\begin{proof}[Proof of Lemma \ref{lemma:bound_third_term}]
We first bound the first term 
\begin{align*}
\sum_{t=1}^T 2 e^{-\frac{t-1}{\tau}} &\leq 2(1 + \int_{1}^\infty e^{-\frac{x-1}{\tau}} dx)  \leq 2 (1 + \tau).
\end{align*}

We now bound the second term, let $\alpha=2\eta \left( 1 + \frac{1}{\eta} G_R\right)$. We have that 
\begin{align*}
\sum_{t=1}^T \sum_{\theta=0}^{t-1} e^{-\frac{\theta}{\tau}} \Vert \mu^{\pi_{t-\theta}} - \mu^{\pi_{t-(\theta+1)}} \Vert_1  &= \alpha \sum_{t=1}^T \sum_{\theta=0}^{t-1} e^{-\frac{\theta}{\tau}} \frac{1}{t-\theta} \\
& = \alpha \left[ e^{-\frac{0}{\tau}} \sum_{t=1}^T 1/t + e^{-\frac{1}{\tau}} \sum_{t=1}^{T-1} 1/t + e^{-\frac{2}{\tau}} \sum_{t=1}^{T-2} 1/t  + ...\right]\\
&\leq \alpha \left[ e^{-\frac{0}{\tau}} \sum_{t=1}^T 1/t + e^{-\frac{1}{\tau}} \sum_{t=1}^{T} 1/t + e^{-\frac{2}{\tau}} \sum_{t=1}^{T} 1/t  + ...\right]\\
&\leq \alpha \left[ \sum_{\theta=0}^T e^{-\frac{\theta}{\tau}}(1 + \ln(T)) \right] \quad \text{since $\sum_{t=1}^T \frac{1}{T} \leq 1+\ln(T)$}\\
&\leq \alpha (1+\ln(T)) (1 + \int_{0}^\infty e^{-\frac{\theta}{\tau}} d\theta)\\
&= \alpha (1+\ln(T)) (1 + \tau).
\end{align*}
\end{proof}

We are now ready to prove Theorem \ref{theorem:mdp_rftl}.
\begin{proof}[Proof of Theorem \ref{theorem:mdp_rftl}]
Combining Eq \eqref{regret_decomposition}, Lemma~\ref{lemma:first_term}, Lemma~\ref{lemma:third_term} and Lemma~\ref{lemma:bound_third_term},
we have 
{
\medmuskip=-0.5mu
\begin{align}
&\quad \sup_{\pi \in \Pi}R(T,\pi) \nonumber \\
&\leq (2\tau + 2) + \left[ \max_{\pi \in \Delta_M}\sum_{t=1}^T\langle \mu^\pi, r_t\rangle - \sum_{t=1}^T \langle \mu^{\pi_t}, r_t \rangle \right] + \left[ \sum_{t=1}^T 2 e^{-\frac{t-1}{\tau}} + \sum_{t=1}^T \sum_{\theta=0}^{t-1} e^{-\frac{\theta}{\tau}} \Vert \mu^{\pi_{t-\theta}} - \mu^{\pi_{t-(\theta+1)}} \Vert_1 \right]\nonumber \\
& \leq 4(\tau + 1) + \left[ \max_{\pi \in \Delta_M}\sum_{t=1}^T\langle \mu^\pi, r_t\rangle - \sum_{t=1}^T \langle \mu^{\pi_t}, r_t \rangle \right]  +2\eta \left( 1 + \frac{1}{\eta} G_R\right) (1+\ln(T)) (1 + \tau). \label{eq:regret-proof}
\end{align}
}
The second term in Eq~\eqref{eq:regret-proof} is bounded by
{
\medmuskip=1mu
\begin{align*}
&\quad \max_{\pi \in \Delta_M}\sum_{t=1}^T\langle \mu^\pi, r_t\rangle - \sum_{t=1}^T \langle \mu^{\pi_t}, r_t \rangle \\
&\leq \max_{\pi \in \Delta_{M,\delta}}\sum_{t=1}^T\langle \mu^\pi, r_t\rangle - \sum_{t=1}^T \langle \mu^{\pi_t}, r_t \rangle + 2 \delta T \left(\vert S\vert \vert A\vert - 1\right) \quad \text{Lemma \ref{lemma:shrinked_loss}}\\
&\leq \sum_{t=1}^T \langle \mu^{\pi_{t+1}} , r_t  \rangle  - \sum_{t=1}^T \langle \mu^{\pi_t}, r_t \rangle+ \frac{T}{\eta} \max_{\mu_1,\mu_2\in \Delta_{M,\delta}} \left[ R(\mu_1) - R(\mu_2) \right]  + 2 \delta T \left(\vert S\vert \vert A\vert - 1\right) \; \text{Lemma \ref{lemma:btl_ftl}}\\
&\leq \sum_{t=1}^T \langle  \mu^{\pi_{t+1}} , r_t  \rangle  - \sum_{t=1}^T \langle \mu^{\pi_t}, r_t \rangle+ \frac{T}{\eta} \ln(\vert S\vert \vert A\vert)   + 2 \delta T \left(\vert S\vert \vert A\vert - 1\right) \quad \text{by choice of function $R$}\\
&\leq \sum_{t=1}^T \Vert r_t\Vert_\infty  \Vert \mu^{\pi_{t+1}}-\mu^{\pi_t} \Vert_1+ \frac{T}{\eta} \ln(\vert S\vert \vert A\vert)  + 2 \delta T \left(\vert S\vert \vert A\vert - 1\right) \quad \text{by Cauchy-Schwarz inequality}\\
&\leq \sum_{t=1}^T\frac{2\eta}{t} \left( 1 + \frac{1}{\eta} G_R\right)+ \frac{T}{\eta} \ln(\vert S\vert \vert A\vert)  + 2 \delta T \left(\vert S\vert \vert A\vert - 1\right) \quad \text{by Lemma \ref{lemma:close_iterates}}\\
&\leq 2\eta \left(1+\frac{1}{\eta} G_R\right) (1 +\ln(T))+ \frac{T}{\eta} \ln(\vert S\vert \vert A\vert)  + 2 \delta T \left(\vert S\vert \vert A\vert - 1\right).
\end{align*}
}
Plugging this result in Eq~\eqref{eq:regret-proof}, we get
\begin{align*}
\sup_{\pi \in \Pi}R(T,\pi)
&\leq 4(\tau + 1) +2\eta (1+\frac{1}{\eta} G_R) (1+\ln(T))+ \frac{T}{\eta} \ln(\vert S\vert \vert A\vert)  \\
&\qquad + 2 \delta T \left(\vert S\vert \vert A\vert - 1\right)  +2\eta ( 1 + \frac{1}{\eta} G_R) (1+\ln(T)) (1 + \tau)\\
&\leq 4(\tau + 1) +4\eta ( 1 + \frac{1}{\eta} G_R) (1+\ln(T)) (1 + \tau) + \frac{T}{\eta} \ln(\vert S\vert \vert A\vert)  + 2 \delta T \left(\vert S\vert \vert A\vert - 1\right)  \\
& = O \left( \tau + 4 \sqrt{\tau T  \ln(\vert S \vert \vert A \vert)} \ln(T) + \sqrt{\tau T  \ln(\vert S \vert \vert A \vert)} + e^{-\frac{\sqrt{T}}{\sqrt{\tau}}} T \vert S \vert \vert A \vert \right).
\end{align*}
Choosing $\eta = \sqrt{\frac{T \ln(\vert S \vert \vert A \vert)}{\tau}}$ and $\delta = e^{-\frac{\sqrt{T}}{\sqrt{\tau}}}$, and using the fact that $G_R\leq \max\{\vert \ln(\delta)\vert, 1\}$, we complete the proof.
\end{proof}

\section{Proof of Theorem~\ref{thm:large_mdp_regret}}\label{sec:regret_analysis_large_mdp}

Using Lemma~\ref{lemma:first_term} and Lemma \ref{lemma:third_term} in Section \ref{sec:regret_analysis}, we can obtain a bound on $\Phi$-MDP-Regret  as follows.
{
\medmuskip=0mu
\begin{align}
&\max_{\pi \in \Pi^\Phi} R(\pi,T) \leq \mathbb{E} \left[ (2\tau + 2) + \max_{\pi \in \Pi^\Phi} \left[ \sum_{t=1}^T \rho^\pi_t -\sum_{t=1}^T \rho_t \right] + \left[ \sum_{t=1}^T \rho_t - \mathbb{E}[ \sum_{t=1}^T r_t(s_t,a_t) ] \right] \right] \nonumber\\
=& \mathbb{E} \left[ (2\tau + 2)  + \left[ \max_{\mu \in \Delta^\Phi_{M,\delta}} \sum_{t=1}^T \langle \mu , r_t \rangle - \sum_{t=1}^T  \langle \mu^{\Phi \tilde{\theta}_t} , r_t \rangle \right] + \left[ \sum_{t=1}^T \rho_t - \mathbb{E}[ \sum_{t=1}^T r_t(s_t,a_t) ] \right] \right]\nonumber\\
\leq& \mathbb{E} \left[  2(2\tau + 2)  + \left[ \max_{\mu \in \Delta^\Phi_{M,\delta}} \sum_{t=1}^T \langle \mu , r_t \rangle - \sum_{t=1}^T  \langle \mu^{\Phi \tilde{\theta}_t} , r_t \rangle \right] \nonumber + \left[ \sum_{t=1}^T \sum_{i=0}^{t-i} e^{-\frac{i}{\tau}} \Vert \mu^{\Phi \tilde{\theta}_{t-i}} - \mu^{\Phi \tilde{\theta}_{t-(i+1)}} \Vert_1 \right] \right]. \label{reg_bound_1}
\end{align}
}
Let $\theta_t^*$ be a solution to the following optimization problem:
\begin{align*}
\max_{\theta \in \Theta}\; & \sum_{i=1}^{t-1}  \left[ \langle \mu , r_i \rangle + \frac{1}{\eta} R^\delta(\mu) \right] \\
\text{s.t\quad} & \mu = \Phi \theta\\
&  \sum_{s\in S} \sum_{a\in A} \mu(s,a) P(s' \vert s,a) = \sum_{a\in A} \mu (s',a) \quad \forall s' \in S\\
&\sum_{s\in S} \sum_{a\in A} \mu(s,a) = 1\\
& \mu(s,a)\geq 0 \quad \forall s\in S, \forall a\in A.
\end{align*}

Since $\{ \Phi \theta_t^* \}_{t=1}^T$ are the iterates of \textsc{RFTL}, we can use the regret guarantee of  \textsc{RFTL} to bound  $\max_{\mu \in \Delta^\Phi_{M,\delta}} \sum_{t=1}^T \langle \mu , r_t \rangle - \sum_{t=1}^T  \langle \mu^{\Phi \theta_t^*} , r_t \rangle$. Notice also that $\mu^{\Phi \theta_t^*} = \Phi \theta_t^*$ as $\theta_t^*$ satisfies all the constraints that ensure $\Phi \theta_t^*$ is an occupancy measure. 

In the remainder of the proof, we want to show that the occupancy measures $\mu^{\Phi \tilde{\theta}_t}$ induced by our algorithm's iterates $\Phi \tilde{\theta}_t$ are close to $\mu^{\Phi \theta_t^*}$.
The rest of the analysis is to prove that $\Vert \mu^{\Phi \theta_t^*} - \mu^{\Phi \tilde{\theta}_t}  \Vert_1$ is small. Notice that using the triangle inequality, we can upper bound this distance by
\begin{align*}
\Vert \mu^{\Phi \theta_t^*} - \mu^{\Phi \tilde{\theta}_t}  \Vert_1 &\leq \Vert \mu^{\Phi \theta_t^*} - P_{\Delta_{M,\delta}^\Phi}(\Phi \tilde{\theta}_t)  \Vert_1  +  \Vert P_{\Delta_{M,\delta}^\Phi}(\Phi \tilde{\theta}_t) - \Phi \tilde{\theta}_t  \Vert_1 + \Vert \Phi \tilde{\theta}_t - \mu^{\Phi \tilde{\theta}_t} \Vert_1 \\
& = \Vert \Phi \theta_t^* - P_{\Delta_{M,\delta}^\Phi}(\Phi \tilde{\theta}_t)  \Vert_1  +  \Vert P_{\Delta_{M,\delta}^\Phi}(\Phi \tilde{\theta}_t) - \Phi \tilde{\theta}_t  \Vert_1 + \Vert \Phi \tilde{\theta}_t - \mu^{\Phi \tilde{\theta}_t} \Vert_1. 
\end{align*}
To bound the last term, the following lemma from \cite{abbasi2014linear} will be useful. It relates a vector $\Phi \tilde{\theta}$ which is almost feasible with its occupancy measure. 

\begin{lemma}\label{u_close_mu}[Lemma 2 in \cite{abbasi2014linear}]
Let $u\in \mathbb{R}^{\vert S \vert \vert A \vert}$ be a vector. Let $\mathcal{N}$ be the set of entries $(s,a)$ where $u(s,a)\leq 0$. Assume
\begin{align*}
\sum_{(s,a)}u(s,a) = 1,\quad \sum_{(s,a)\in \mathcal{N}}\vert u(s,a) \vert \leq \epsilon',\quad \Vert u^\top (P-B)\Vert_1 \leq \epsilon''.
\end{align*}
Vector $[u]_+ / \Vert [u]_+\Vert_1$ defines a policy, which in turn defines a stationary distribution $\mu^u$. It holds that 
\begin{align*}
\Vert \mu^u - u \Vert_1 \leq \tau \ln(\frac{1}{\epsilon'})(2\epsilon' + \epsilon'') + 3\epsilon' .
\end{align*}
\end{lemma}

Suppose we are given a vector $\Phi \tilde{\theta}_t$ such that $\Vert [\Phi \tilde{\theta}_t]_{(\delta,-)} \Vert_1 \leq \epsilon'$ and $\Vert (\Phi \tilde{\theta_t})^\top (P-B)\Vert_1 \leq \epsilon''$. In view of Lemma \ref{u_close_mu} and the fact that $\Vert [\Phi \tilde{\theta}_t]_- \Vert_1 \leq \Vert [\Phi \tilde{\theta}_t]_{(\delta,-)} \Vert_1 \leq \epsilon'$,  we have a bound on $\Vert \Phi \tilde{\theta}_t - \mu^{\Phi \tilde{\theta}_t} \Vert_1$. The next lemma shows that we can also obtain a bound on $\Vert P_{\Delta_{M,\delta}^\Phi}(\Phi \tilde{\theta}_t) - \Phi \tilde{\theta}_t  \Vert_1$. 

\begin{lemma}\label{phi_theta_close_to_projection}
Let $\Phi \tilde{\theta}_t$ be a vector such that $\Vert [\Phi \tilde{\theta}]_{(\delta,-)} \Vert_1 \leq \epsilon'$ and $\Vert (\Phi \tilde{\theta})^\top (P-B)\Vert_1 \leq \epsilon''$ for some $\epsilon', \epsilon'' \geq 0$. It holds that 
\begin{align*}
\Vert P_{\Delta_{M,\delta}^\Phi}(\Phi \tilde{\theta}_t) - \Phi \tilde{\theta}_t  \Vert_1 \leq c (\epsilon' + \epsilon''),
\end{align*}
where $c$ is a bound on the $l_\infty$ norm of the Lagrange multipliers of certain linear programming problem.
\end{lemma}

\begin{proof}
The idea comes from sensitivity analysis in Linear Programming (LP) (see for example \cite{schrijver1998theory}). Consider the $l_1$ projection problem of $\Phi \tilde{\theta}_t$ onto the set of occupancy measures parametrized by $\Phi$
\begin{align*}
\min_{\theta} &\Vert \mu - \Phi \tilde{\theta} \Vert_1 \\
\text{s.t} \quad & \mu = \Phi \theta\\
& \mu^\top 1 = 1\\
&\mu \geq \delta\\
&\mu^\top(P-B) = 0\\
& \theta \in \Theta.
\end{align*}
It can be reforumulated as the following LP
\begin{align*}
\text{Primal 1:}\qquad \min_{\theta,u} &\sum_{(s,a)} u(s,a) \\
\text{s.t} \quad & u(s,a) - [\Phi \theta](s,a) \geq -[\Phi \tilde{\theta}](s,a)\\
& u(s,a) + [\Phi \theta](s,a) \geq [\Phi \tilde{\theta}](s,a)\\
& \mu = \Phi \theta\\
& \mu^\top 1 = 1\\
&\mu \geq \delta\\
&\mu^\top(P-B) = 0\\
& - \theta(i) \geq  - W \quad \forall i=1,...,d\\
& \theta(i) \geq  0  \quad  \forall i=1,...,d\\
\end{align*}
Let us now consider the perturbed problem `Primal 2' which arises by perturbing the right hand side vector of `Primal 1':
\begin{align*}
\text{Primal 2:}\qquad \min_{\theta,u} &\sum_{(s,a)} u(s,a) \\
\text{s.t} \quad & u(s,a) - [\Phi \theta](s,a) \geq -[\Phi \tilde{\theta}](s,a)\\
& u(s,a) + [\Phi \theta](s,a) \geq [\Phi \tilde{\theta}](s,a)\\
& \mu = \Phi \theta\\
& \mu^\top 1 = 1\\
&\mu \geq \delta + \tilde{a}\\
&\mu^\top(P-B) = \tilde{b}\\
& - \theta(i) \geq  - W \quad \forall i=1,...,d\\
& \theta(i) \geq  0  \quad  \forall i=1,...,d\\
\end{align*}
We choose perturbation vectors $\tilde{a},\tilde{b}$ such that the optimal value of ` Primal 2' is zero is 0.
Let $b$ be the right hand side vector of `Primal 1' and $b'\triangleq b - \xi$ be that of `Primal 2' for some vector $\xi$. Since by assumption we have that $\Vert [\Phi \tilde{\theta}]_{(\delta,-)} \Vert_1 \leq \epsilon'$ and $\Vert (\Phi \tilde{\theta})^\top (P-B)\Vert_1 \leq \epsilon''$ then it holds that $\Vert b-b' \Vert_1 = \Vert \xi \Vert_1 \leq \epsilon' + \epsilon ''$. Let `Opt. Primal 1' and `Opt. Primal 2' be the optimal value of the respective problems (`Opt. Primal 2' = 0 by construction) and let $\lambda^*$ be the vector of optimal dual variables of `Dual 1', the problem dual to `Primal 1'. Since by assumption, the feasible set of `Primal 1' is feasible, then the absolute value of the entries of  $\lambda^*$ is bounded by some constant $c$. 

Now, since $\lambda^*$ is feasible for `Dual 2', the following sequence of inequalities hold:
\begin{align*}
&\text{`Opt. Primal 2'} \geq (\lambda^*)^\top (b - \xi)\\
\iff &  \text{`Opt. Primal 2'} \geq  \text{`Opt. Primal 1'} - (\lambda^*)^\top \xi.
\end{align*}
Therefore,
\begin{align*}
\text{`Opt. Primal 1'} &\leq \text{`Opt. Primal 2'} + \Vert  \lambda^* \Vert_\infty \Vert \xi \Vert_\infty \\
& = 0 +  \Vert  \lambda^* \Vert_\infty \Vert \xi \Vert_1\\
&\leq c (\epsilon' + \epsilon''),
\end{align*}
which yields the result.
\end{proof}

Now we proceed to bound $\Vert \Phi \theta_t^* - P_{\Delta_{M,\delta}^\Phi}(\Phi \tilde{\theta}_t) \Vert_1 $. Consider the function
\begin{align}
F_t( \Phi \theta) \triangleq \sum_{i = 1}^t [\langle r_i, \Phi \theta \rangle - \frac{1}{\eta} R^\delta(\Phi \theta)] 
\end{align}
Since $R^\delta$ is strongly convex over $\Delta_{M,\delta}^\Phi$ with respect to $\Vert \cdot \Vert_1$ (but not everywhere over the reals as the extension uses a linear function), we have that $F_t$ is $\frac{t}{\eta}$-strongly concave with respect to $\Vert \cdot \Vert_1$ over $\Delta_{M,\delta}^\Phi$. With this in mind we can prove the following result.

\begin{lemma}\label{star_close_to_proj_tilde}
Let $\Phi \tilde{\theta}_{t+1}$ be a vector such that $\Vert [\Phi \tilde{\theta}_{t+1}]_{(\delta,-)} \Vert_1 \leq \epsilon'$ and $\Vert (\Phi \tilde{\theta}_{t+1})^\top (P-B)\Vert_1 \leq \epsilon''$ for some $\epsilon', \epsilon'' \geq 0$. Let $\epsilon'''$ be such that $F_{t}(\Phi \theta_{t+1}^*) - F_{t}(\Phi \tilde{\theta}_{t+1}) \leq \epsilon'''$. And let $G_{F_{t}}$ be the Lipschitz constant of $F_{t}$ with respect to norm $\Vert \cdot \Vert_1 $ over the set $\Delta_{M,\delta}^\Phi$.  It holds that 

\begin{align*}
\Vert \Phi \theta_{t+1}^* - P_{\Delta_{M,\delta}^\phi}(\Phi \tilde{\theta}_{t+1}) \Vert_1 \leq \sqrt{\frac{2 \eta}{t } (\epsilon''' + G_{F_t} c (\epsilon' + \epsilon''))}.
\end{align*}

\end{lemma}

\begin{proof}
Since $F_t$ is $\frac{t}{\eta}$-strongly concave over $\Delta_{M,\delta}^\Phi$ and $\Phi \theta_{t+1}^*$ is the optimizer of $F_t$ over $\Delta_{M,\delta}^\Phi$. It holds that 
\begin{align*}
\frac{t}{2\eta} \Vert \Phi \theta_{t+1}^* - \Phi \tilde{\theta}_{t+1} \Vert_1^2 &\leq F_t (\Phi \theta_{t+1}^*) - F_t( P_{\Delta_{M,\delta}^\Phi} (\Phi \tilde{\theta}_{t+1}))  \\
& \leq F_t (\Phi \theta_{t+1}^*) - F_t( \Phi \tilde{\theta}_{t+1}) + G_{F_t} \Vert P_{\Delta_{M,\delta}^\Phi} (\Phi \tilde{\theta}_{t+1}) - \Phi \tilde{\theta}_{t+1} \Vert_1\\
&\leq \epsilon''' + G_{F_t} \Vert P_{\Delta_{M,\delta}^\Phi} (\Phi \tilde{\theta}_{t+1}) - \Phi \tilde{\theta}_{t+1} \Vert_1 \quad \text{by assumption}\\
&\leq \epsilon''' + G_{F_t}c(\epsilon' + \epsilon'') \quad \text{by Lemma \ref{phi_theta_close_to_projection}}
\end{align*}
which yields the result.

\end{proof}

The next lemma bounds the Lipschitz constant $G_{F_t}$.

\begin{lemma}\label{bound_G_F_t}
Let $\eta = \sqrt{\frac{T}{\tau}}$, $\delta = e^{-\sqrt{T}}$. The function $F_t (\mu): \mathbb{R}^{\vert S\vert \vert A\vert } \rightarrow \mathbb{R}$ is $G_{F_t}$-Lipschitz continuous on variables $\mu$ with respect to norm $\Vert \cdot \Vert_1$ over $\Delta_{M,\delta}^\Phi$ with $G_{F_t} \leq  t(1+2\sqrt{\tau} d W)$.  
\end{lemma}

\begin{proof}
It suffices to find a an upper bound for $\Vert \nabla_\mu F_t(\mu) \Vert_\infty$. Since $\nabla_{\mu} F_t(\mu) = \sum_{i=1}^t r_i - \frac{t}{\eta} \nabla_\mu R^\delta(\mu)$, we have that 
\begin{align*}
\Vert \nabla_\mu F_t(\mu) \Vert_\infty &\leq \Vert \sum_{i=1}^t r_i \Vert_\infty + \frac{t}{\eta} \Vert \nabla_\mu R^\delta(\mu)\Vert_\infty \quad \text{by triangle inequality}\\
&\leq t + \frac{t}{\eta} \Vert \nabla_\mu R^\delta(\mu)\Vert_\infty \quad \text{since $\vert r_i(s,a)\vert \leq 1$}\\
&\leq t + \frac{t}{\eta}\max\{\vert 1+ \ln(\delta)\vert, \vert 1 +\ln(dW) \vert \} \quad \text{as in the proof of Lemma \ref{lemma:entropy_lipschitz} }. 
\end{align*}
The second to last inequality holds since $\vert \frac{d}{dx} x\ln(x) \vert = \vert 1 + \ln(x) \vert$ and the maximum will occur at $x=\delta$ or $x=[\Phi \theta](s,a)$, $[\Phi \theta](s,a)$ can be bounded by $Wd$. Plugging in the values for $\eta$ and $\delta$ we get 
\begin{align*}
\Vert \nabla_\mu F_t(\mu) \Vert_\infty &\leq t + \frac{t\tau}{\sqrt{T}}(1 + \max\{\sqrt{T}, dW\})\\
&\leq t + \frac{t\tau}{\sqrt{T}} (2\sqrt{T} W d)\\
&= t(1+2\sqrt{\tau} d W).
\end{align*}
\end{proof}

Combining the previous three lemmas, we obtain the following result.

\begin{lemma}\label{dist_mu_star_mu_tilde}
Let $\Phi \tilde{\theta}_{t+1}$ be a vector such that $\Vert [\Phi \tilde{\theta}_{t+1}]_{(\delta,-)} \Vert_1 \leq \epsilon'$ and $\Vert (\Phi \tilde{\theta}_{t+1})^\top (P-B)\Vert_1 \leq \epsilon''$ for some $\epsilon', \epsilon'' \geq 0$. Let $\epsilon'''$ be such that $F_{t}(\Phi \theta_{t+1}^*) - F_{t}(\Phi \tilde{\theta}_{t+1}) \leq \epsilon'''$. And let $G_{F_{t}}$ be the Lipschitz constant of $F_{t}$ with respect to norm $\Vert \cdot \Vert_1 $ over the set $\Delta_{M,\delta}^\Phi$. It holds that 

\begin{align*}
\Vert \mu^{\Phi \theta_t^*} - \mu^{\Phi \tilde{\theta}_t}  \Vert_1 &\leq \tau \ln(\frac{1}{\epsilon'})(2\epsilon' + \epsilon'') + 3\epsilon' + c(\epsilon' + \epsilon'') + \sqrt{\frac{2 \eta}{t } (\epsilon''' + G_{F_t} c (\epsilon' + \epsilon''))}.
\end{align*}

\end{lemma}

\begin{proof}
By triangle inequality, we have
\begin{align*}
\Vert \mu^{\Phi \theta_t^*} - \mu^{\Phi \tilde{\theta}_t}  \Vert_1 &\leq  \Vert \Phi \theta_t^* - P_{\Delta_{M,\delta}^\Phi}(\Phi \tilde{\theta}_t)  \Vert_1  +  \Vert P_{\Delta_{M,\delta}^\Phi}(\Phi \tilde{\theta}_t) - \Phi \tilde{\theta}_t  \Vert_1 + \Vert \Phi \tilde{\theta}_t - \mu^{\Phi \tilde{\theta}_t} \Vert_1. 
\end{align*}
Using Lemmas \ref{u_close_mu}, \ref{phi_theta_close_to_projection}, and \ref{star_close_to_proj_tilde} to bound the first, second, and third terms respectively yields the result.

\end{proof}

Now we can upper bound the bound on the $\Phi$-MDP-Regret, Eq \eqref{reg_bound_1}, using triangle inequality and Lemma \ref{dist_mu_star_mu_tilde}. For the bound to be useful we want to be able to produce vectors $\{\Phi \tilde{\theta}_t\}_{t=1}^T$ that satisfy the conditions of Lemma \ref{dist_mu_star_mu_tilde} with $\epsilon', \epsilon'', \epsilon'''$ that are small enough. It is also important that we produce $\{\Phi \tilde{\theta}_t\}_{t=1}^T$ in a computationally efficient manner. At time $t$, our approach to generate  $\Phi \tilde{\theta}_t$, will be to run Projected Stochastic Gradient Descent on function \ref{eq:c_t_eta}.
The following theorem from \cite{abbasi2014linear} will be useful.  

\begin{theorem}[Theorem 3 in \cite{abbasi2014linear}]\label{sgd_guarantee}
Let $\mathcal{Z}\subset \mathbb{R^d}$ be a convex set such that $\Vert z\Vert_2 \leq Z$ for all $z \in \mathcal{Z}$ for some $Z> 0$. Let $f$ be a concave function defined over $\mathcal{Z}$. Let $\{z_k\}_{k=1}^K \in \mathcal{Z}^T$ be the iterates of Projected Stochastic Gradient Ascent, i.e. $z_{k+1} \leftarrow P_\mathcal{Z}(x_k + \eta f_t')$ where $P_\mathcal{Z}$ is the euclidean projection onto $\mathcal{Z}$, $\eta$ is the step-size and $\{f_k'\}_{k=1}^K$ are such that $\mathbb{E}[f_k' \vert z_k] = \nabla f(z_k)$ with $\Vert f_k'\Vert_2 \leq F$ for some $F>0$. Then, for $\eta = \frac{Z}{(F\sqrt{K})}$ for all $\kappa \in (0,1)$, with probability at least $1-\kappa$ it holds that
\begin{align*}
\max_{z\in \mathcal{Z}} f(z) - f( \frac{1}{K} \sum_{k=1}^K z_k) \leq \frac{Z F}{\sqrt{K}} + \sqrt{\frac{\left(1+4Z^2 K\right) \left( 2 \ln(\frac{1}{\kappa}) + d \ln(1 + \frac{Z^2 T}{d})\right)}{K^2}}.
\end{align*} 
\end{theorem}

In view of Theorem \ref{sgd_guarantee} we need to design a stochastic subgradient for $c^{t,\eta}$ and a bound for its $l$-2 norm. We follow the approach in \cite{abbasi2014linear}, we notice however that the objective function considered in \cite{abbasi2014linear} does not contain the regularizer $R^\delta$ so must take care of that in our analysis.

Lemma \ref{stochastic_gradient_of_c} creates a stochastic subgradient for $c^{t,\eta}$ and provides an upper bound for its $l$-2 norm. We now present its proof.



\begin{proof}[Proof of Lemma \ref{stochastic_gradient_of_c}]
Let us first compute $\nabla_\theta c^{\eta,t} (\theta)$. Define $r_{:t}\triangleq \sum_{i=1}^t r_i$
By definition we have 
\begin{align*}
c^{\eta,t}(\theta) &= (\Phi \theta)^\top r_{:t} - \frac{t}{\eta} \sum_{(s,a)} R^\delta_{(s,a)}(\Phi \theta) - H_t \Vert [\Phi \theta]_(\delta,-)\Vert_1 -H_t \Vert (P-B)^\top \Phi \theta \Vert_1\\
&= \theta^\top (\Phi^\top r_{:t}) - \frac{t}{\eta} \sum_{(s,a)} R^\delta_{(s,a)}(\Phi \theta) - H_t \sum_{(s,a)}[\Phi_{(s,a),:}\theta]_{(\delta,-)} - H_t \sum_{s} \vert [(P-B)^\top \Phi]_{s,:} \theta \vert.
\end{align*}
So, we get
\begin{align*}
\nabla_\theta c^{t,\eta}(\theta) &= \Phi^\top r_{:t} - \frac{t}{\eta} \sum_{(s,a)} \nabla_\theta R^\delta_{(s,a)}(\Phi \theta)\\
& - H_t \sum_{(s,a)} - \Phi_{(s,a),:} \mathbb{I}\{\Phi_{(s,a),:} \theta \leq \delta \}  - H_t \sum_{s} [(P-B)^\top \Phi]_{s,:} sign([(P-B)^\top \Phi ]_{s,:} \theta)
\end{align*}
We design a stochastic subgradient $g$ of $\nabla_\theta c^{\eta, t}(\theta)$ by sampling a state-action pair $(s',a')$ from the given distribution $q_1$ and a state $s''$ from distribution $q_2$. 
Then, we have
\begin{align*}
g_{s',a',s''}(\theta) &= \Phi^\top r_{:t} + \frac{H_t}{q_1(s',a')} \Phi_{(s',a'),:} \mathbb{I}\{\Phi_{(s',a'),:} \leq \delta \}\\
&\quad - \frac{H_t}{q_2(s'')} [(P-B)^\top \Phi]_{s'',:} sign([(P-B)^\top \Phi ]_{s'',:} \theta) - \frac{t}{\eta q_1(s',a')} \nabla_\theta R^\delta_{(s',a')}(\Phi \theta).
\end{align*}
We will also give a closed form expression of $\nabla_\theta R^\delta_{(s',a')}(\Phi \theta)$ in the proof below. By construction, it holds that $\mathbb{E}_{(s',a')\sim q_1, s'' \sim q_2}[g_{s',a',s''}(\theta) \vert \theta] = \nabla_\theta c^{t,\eta}(\theta)$. To simplify notation let $g(\theta) = g_{s',a',s''}(\theta)$.

We now bound $\Vert g(\theta)\Vert_2$ with probability 1. First, we have
\begin{align*}
\Vert \Phi^\top r_{:t}\Vert_2 &= \sqrt{\sum_{i=1}^d (r_{:t}^\top \Phi_{:,i})^2}\\
&\leq \sqrt{\sum_{i=1}^d (\Vert r_{:t} \Vert_\infty \Vert \Phi_{:,i}\Vert_1)^2}\quad \text{by Cauchy-Schwarz}\\
&\leq \sqrt{d t^2 1} = t\sqrt{d},
\end{align*}
where the last inequality holds since $\Vert r_i \Vert_\infty \leq 1$ for $t=1,...,T$ and each column of $\Phi$ is a probability distribution.
Next, we have
\begin{align*}
&\left\Vert \frac{H_t}{q_1(s',a')} \Phi_{(s',a'),:} \mathbb{I}\{ \Phi_{(s',a'),:} \theta \leq \delta \}\right\Vert_2 \leq H_t C_1,\quad\text{and}\\
&\left\Vert - \frac{H_t}{q_2(s'')}  [(P-B)^\top \Phi]_{s'',:} sign([(P-B)^\top \Phi ]_{s'',:} \theta) \right\Vert_2 \leq H_t C_2,
\end{align*}
where $C_1$ and $C_2$ are defined in (\ref{def_of_Cs}).
Finally, we bound $\Vert \nabla_\theta R^\delta_{(s,a)}(\Phi \theta)\Vert_2$. By definition of $R^\delta_{(s,a)}$ in Eq~\ref{eq:def_R_delta}, we need to compute the gradients of the negative entropy function $\nabla_\theta R(\Phi\theta)$. Let us compute $\frac{d}{d\theta_i}R(\Phi \theta)$.
\begin{align*}
\frac{d}{d\theta_i}R(\Phi \theta) &= \sum_{(s,a)} \frac{d}{d\theta_i} R_{(s,a)}(\Phi \theta)\\
&= \sum_{(s,a)} \frac{d}{d\theta_i} \left[ (\sum_{k=1}^d \Phi_{(s,a),k} \theta_k) \ln(\sum_{k=1}^d \Phi_{(s,a),k} \theta_k) \right]\\
& = \sum_{(s,a)} (\sum_{k=1}^d \Phi_{(s,a),k} \theta_k) (\frac{d}{d\theta_i} \ln(\sum_{k=1}^d \Phi_{(s,a),k} \theta_k)) + \ln(\sum_{k=1}^d \Phi_{(s,a),k} \theta_k)\Phi_{(s,a),i}\\
& = \sum_{(s,a)} (\sum_{k=1}^d \Phi_{(s,a)} \theta_k) \frac{1}{\sum_{k=1}^d \Phi_{(s,a)} \theta_k} \frac{d}{d\theta_i} (\sum_{k=1}^d \Phi_{(s,a),k} \theta_k ) + \ln(\sum_{k=1}^d \Phi_{(s,a),k} \theta_k)\Phi_{(s,a),i}\\
& =  \sum_{(s,a)} \Phi_{(s,a),i} + \ln(\sum_{k=1}^d \Phi_{(s,a),k} \theta_k)\Phi_{(s,a),i}.
\end{align*}

We are also interested in the gradient of the linear extension of $R_{(s,a)}(x)$: $R_{(s,a)}(\delta) +  \frac{d}{d x} R_{(s,a)}(\delta) (x-\delta)$ which is equal to $\delta \ln(\delta) +  (1+\ln(\delta)) (x-\delta)$. So we upper bound $\vert \frac{d}{d \theta_i} \delta \ln(\delta) + (1+\ln(\delta))(\Phi_{(s,a),:}\theta - \delta)\vert$ 
\begin{align*}
&\vert \frac{d}{d \theta_i} \delta \ln(\delta) + (1+\ln(\delta))(\Phi_{(s,a),:}\theta - \delta)\vert \\
=&  \vert \frac{d}{d \theta_i} (1+\ln(\delta))(\Phi_{(s,a),:}\theta - \delta)\vert \\
=& \vert (1+\ln(\delta)) \Phi_{(s,a),i}\vert.
\end{align*}

It follows that 
\begin{align*}
&\Vert \nabla_\theta  R^\delta_{(s,a)(\Phi \theta)}\Vert_2\\
\leq &\left( \sum_{i=1}^d \left[\max \{ \Phi_{(s,a),i} + \ln (W \sum_{k=1}^d \Phi_{(s,a),k}) \Phi_{(s,a),i},\vert (1+\ln(\delta)) \Phi_{(s,a),i} \vert \} \right]^2\right)^{1/2}\\
\leq&  \left( \sum_{i=1}^d \left[ (1 + \max \{ \ln (W \sum_{k=1}^d \Phi_{(s,a),k}), \vert \ln(\delta) \vert \} \Phi_{(s,a),i} \right]^2\right)^{1/2}\\
\leq& \left( \sum_{i=1}^d \left[ (1 + \max \{ \ln (Wd), \vert \ln(\delta) \vert \} \Phi_{(s,a),i} \right]^2\right)^{1/2}\\
\leq&   (1 + \ln (Wd) + \vert \ln(\delta) \vert) \Vert \Phi_{(s,a),:}\Vert_2.
\end{align*}
Thus $\Vert \frac{t}{\eta q_1(s',a')} \nabla_\theta R^\delta_{(s',a')}(\Phi \theta)\Vert_2 \leq \frac{t}{\eta}(1 + \ln (Wd) + \vert \ln(\delta) \vert) C_1$. Using triangle inequality we have that with probability 1
\begin{align*}
\Vert g(\theta)\Vert_2 \leq t \sqrt{d} + H(C_1 + C_2) + \frac{t}{\eta}(1 + \ln (Wd) + \vert \ln(\delta) \vert) C_1.
\end{align*}

\end{proof}

By using Lemma~\ref{stochastic_gradient_of_c} and the fact that since $\theta \in \Theta$ then $\Vert \theta \Vert_2 \leq W$. We can prove the following.

\begin{lemma}\label{sgd_guarantee_on_function_c}
For all $t=1,...,T$, $\eta>0$, $\kappa \in (0,1)$, after running $K(t)$ iterations of Projected Stochastic Gradient Ascent on function $c^{\eta,t}(\theta)$ over the set $\Theta^\Phi$ and using step-size $\frac{W}{\sqrt{K(t)}G'}$ with $G' = t \sqrt{d} + H_t (C_1 + C_2) + \frac{t}{\eta}(1 + \ln (Wd) + \vert \ln(\delta) \vert) C_1$ with probability at least $1-\kappa$ it holds that
\begin{align*}
&\sum_{i=1}^t \left[ \langle r_i, \Phi \theta_{t+1}^* \rangle - \frac{1}{\eta} R^\delta (\Phi \theta_{t+1}^*)\right]\\
&- \left[ \sum_{i=1}^t \left[ \langle r_i, \Phi \tilde{\theta}_{t+1} \rangle - \frac{1}{\eta} R^\delta (\Phi \tilde{\theta}_{t+1})\right] - H_t \Vert (\Phi \tilde{\theta}_{t+1})^\top (P-B) \Vert_1 -H_t \Vert [\Phi \tilde{\theta}_{t+1}]_{(\delta,-)} \Vert_1 \right]\\
&\leq \frac{W G'}{\sqrt{K(t)}} + \sqrt{\frac{(1+4S^2 K(t))(2\ln(\frac{1}{\kappa}) + d \ln(1 + \frac{W^2 K(t)}{d}))}{K(t)^2}}.
\end{align*}
\end{lemma}

\begin{proof}
The proof follows from applying Theorem \ref{sgd_guarantee} on function $c^{\eta,t}(\theta)$. Using the bound of the stochastic gradients from Lemma \ref{stochastic_gradient_of_c}, as well as the fact that 
$
\max_{\theta \in \Theta^\Phi}c^{\eta,t}(\theta) \geq c^{\eta,t}(\tilde{\theta}_{t+1})
$
and since $\Phi \tilde{\theta}_{t+1}$ is feasible, we have  $\Vert (\Phi \theta_{t+1}^*)^\top (P-B) \Vert_1 = 0$ and $ \Vert [\Phi \theta_{t+1}^*]_{(\delta,-)} \Vert_1=0$.
\end{proof}
We remark that we did not relax the constraint $(\Phi \theta)^\top 1 = 1$ and in fact when we use Projected Gradient Ascent we are projecting onto a subset of that hyperplane, although $\Phi$ has $\vert S \vert \vert A \vert$ rows we can precompute the vector $\Phi^\top 1 \in \mathbb{R}^d$ so that all projections to the subset of the hyper plane given by $(\Phi \theta)^\top 1 = 1$ can be done in $O(poly(d))$ time.

The next lemma bounds the largest difference the function $F_t(\Phi \theta)$ can take over $\theta \in \Theta^\Phi$. It will be clear later why this bound is needed

\begin{lemma}\label{diameter_using_F_t}
For all $t=1,...,T$. It holds that 
\begin{align*}
\max_{\theta_1, \theta_2 \in \Theta^\Phi} F_t(\Phi \theta_1) - F_t(\Phi \theta_2) \leq t \left[ 2 + \frac{1}{\eta}(1 + \ln(\vert S \vert  \vert A\vert))\right].
\end{align*}
\end{lemma}

\begin{proof}
By definition of $F_t$ it suffices to bound 
\begin{align*}
\sum_{i=1}^t \langle r_i, \Phi \theta_1 - \Phi \theta_2 \rangle + \frac{t}{\eta} \left[ R^\delta(\Phi \theta_2) - R^\delta(\Phi \theta_1) \right].
\end{align*}

Now, we have
\begin{align*}
&\sum_{i=1}^t \langle r_i, \Phi \theta_1 - \Phi \theta_2 \rangle \leq \sum_{i=1}^t \Vert r_i \Vert_\infty \Vert \Phi \theta_1 - \Phi \theta_2 \Vert_1 \quad \text{By Cauchy-Schwarz}\\
& \leq  \sum_{i=1}^t 1 \Vert \Phi \theta_1 - \Phi \theta_2 \Vert_1\\
& \leq  \sum_{i=1}^t  \Vert \Phi \theta_1\Vert_1 \Vert \Phi \theta_2 \Vert_1 \quad \text{by triangle inequality}\\
& \leq 2t,
\end{align*}
where the last inequality holds since all entries of $\Phi$ and $\theta$ are nonnegative, and $(\Phi \theta)^\top 1 =1$ for all $\theta \in \Theta^\Phi$.

It is well known that the minimizer of $R(\mu)$ for $\mu \in \Delta^{\vert S \vert \vert A\vert}$ is $-\ln(\vert S \vert \vert A\vert)$. Moreover, its optimal solution $\mu^*$ is equal to the vector with value ${1}/(\vert S \vert \vert A\vert)$ on each of its entries, which is of course in the interior of the simplex. Notice that since $R^\delta$ is an extension of $R$, if $\delta$ is sufficiently small (which we ensure by the choice of $\delta$ later in the analysis), the minimizer of $R^\delta(\Phi \theta)$ for $\theta \in \Theta^\Phi$ will be bounded below by $-\ln(\vert S \vert \vert A\vert)$. That is 
\begin{align*}
-\ln(\vert S \vert \vert A\vert) \leq \min_{\theta \in \Theta^\Phi} R^\delta(\Phi \theta).
\end{align*}
We upper bound $\max_{\theta \in \Theta^\Phi}R^\delta (\Phi \theta)$. Since all entries of $\Phi$ and $\theta \in \Theta^\Phi$ are greater than or equal to $0$, we have that $R^\delta(\Phi \theta) \leq R(\Phi \theta)$ for all $\theta \in \Theta^\Phi$. Again, since all entries  of $\Phi$ and $\theta \in \Theta^\Phi$ are greater than or equal to $0$ we have that for any $\theta \in \Theta^\Phi$ 
\begin{align*}
R(\Phi \theta) &\triangleq (\Phi \theta)^\top \ln(\Phi \theta)\\
&\leq \Vert \Phi \theta \Vert_1 \Vert \ln(\Phi \theta) \Vert_\infty\\
&= 1 \Vert \ln(\Phi \theta) \Vert_\infty \quad \text{since $(\Phi \Theta)^\top 1=1$}\\
&\leq \Vert \ln(\Phi \theta) \Vert_1\\
&\leq 1 \quad \text{since $(\Phi \Theta)^\top 1=1$, and for any $x\in \mathbb{R}_+, \ln(x)\leq x$}.
\end{align*}
we have shown that 
\begin{align*}
\sum_{i=1}^t \langle r_i, \Phi \theta_1 - \Phi \theta_2 \rangle + \frac{t}{\eta} \left[ R^\delta(\Phi \theta_2) - R^\delta(\Phi \theta_1) \right] \leq 2t + \frac{t}{\eta}\left[ 1 + \ln( \vert S \vert \vert A\vert )\right]
\end{align*}
which finishes the proof.
\end{proof}

Lemma \ref{dist_mu_star_mu_tilde} assumes we have have at our disposal a vector $\Phi \tilde{\theta}_{t+1}$ such that $\Vert [\Phi \tilde{\theta}_{t+1}]_{(\delta,-)} \Vert_1 \leq \epsilon'$ and $\Vert (\Phi \tilde{\theta}_{t+1})^\top (P-B)\Vert_1 \leq \epsilon''$, and $F_{t}(\Phi \theta_{t+1}^*) - F_{t}(\Phi \tilde{\theta}_{t+1}) \leq \epsilon'''$  for some $\epsilon', \epsilon'', \epsilon''' \geq 0$. We now show how to obtain such error bounds by running at each time step $t$, $K(t)$ iterations of PSGA and using Lemma \ref{sgd_guarantee_on_function_c}.

\begin{lemma}
For $t =1,...,T $, let $b_{K(t)}$ the right hand side of the equation in the bound of Lemma \ref{sgd_guarantee_on_function_c} and assume the same conditions hold. After $K(t)$ iterations of PSGA, with probability at least $1-\kappa$, it holds that
\begin{align*}
\Vert [\Phi \tilde{\theta}_{t+1}]_{(\delta,-)} \Vert_1 &\leq  \frac{1}{H_t} \left[ b_{K(t)} + t [2 + \frac{1}{\eta}(1 + \ln(\vert S \vert \vert A \vert))]\right], \\
 \Vert (\Phi \tilde{\theta}_{t+1})^\top (P-B)\Vert_1 &\leq  \frac{1}{H_t} \left[ b_{K(t)} + t [2 + \frac{1}{\eta}(1 + \ln(\vert S \vert \vert A \vert))]\right], \\
F_{t}(\Phi \theta_{t+1}^*) - F_{t}(\Phi \tilde{\theta}_{t+1}) &\leq b_{K(t)}.
\end{align*}
\end{lemma}

\begin{proof}
To show the first two inequalities, notice that Lemma \ref{sgd_guarantee_on_function_c} implies 
\begin{align*}
\frac{1}{H_t} \Vert [\Phi \tilde{\theta}_{t+1}]_{(\delta,-)} + \frac{1}{H_t}  \Vert (\Phi \tilde{\theta}_{t+1})^\top (P-B)\Vert_1 & \leq b_{K(t)} + F_t(\Phi \tilde{\theta}_{t+1}) - F_t(\Phi \theta^*_{t+1}) \\
& \leq b_{K(t)} + t \left[ 2 + \frac{1}{\eta}(1 + \ln(\vert S \vert  \vert A\vert))\right],
\end{align*}
where the last inequality holds by Lemma \ref{diameter_using_F_t}. Since $\Vert \cdot \Vert_1 \geq 0$ we get the desired results. To show that $F_{t}(\Phi \theta_{t+1}^*) - F_{t}(\Phi \tilde{\theta}_{t+1}) \leq b_{K(t)}$ again use Lemma \ref{sgd_guarantee_on_function_c} and the fact that $\Vert \cdot \Vert_1 \geq 0$.

\end{proof}

We are ready to prove the main theorem from this section. 
\begin{proof}[Proof of Theorem \ref{thm:large_mdp_regret}]
Recall the $\Phi$-MDP-Regret regret bound from Equation \ref{reg_bound_1}.
\begin{align*}
& \max_{\pi \in \Pi^\Phi} R(\pi,T)\\
&\leq \mathbb{E}_{PSGA} [  (4\tau + 4)\\
&+ [ \max_{\mu \in \Delta^\Phi_{M,\delta}} \sum_{t=1}^T \langle \mu , r_t \rangle - \sum_{t=1}^T  \langle \mu^{\Phi \tilde{\theta}_t} , r_t \rangle ]+ [ \sum_{t=1}^T \sum_{i=0}^{t-i} e^{-\frac{i}{\tau}} \Vert \mu^{\Phi \tilde{\theta}_{t-i}} - \mu^{\Phi \tilde{\theta}_{t-(i+1)}} \Vert_1 ] ] \label{reg_bound_1}.
\end{align*}
Since it is cumbersome to work with the $\mathbb{E}_{PSGA}[\cdot]$ in our bounds let us make the following argument. For $t=1,...,T$, define $\mathcal{E}_t$ be the event that the inequality in Lemma \ref{sgd_guarantee_on_function_c} holds, let $\mathcal{E} \triangleq \cap_{t=1}^T \mathcal{E}_t$. For any random variable $X$ we know that $\mathbb{E}_{PSGA}[X] = \mathbb{E}_{PSGA}[X\vert \mathcal{E}] P(\mathcal{E}) + \mathbb{E}_{PSGA}[X\vert \mathcal{E}^c] P(\mathcal{E}^c)$. Let us work conditioned on the event $\mathcal{E}$, we will later bound $\mathbb{E}_{PSGA}[X\vert \mathcal{E}^c] P(\mathcal{E}^c)$.

By triangle inequality, Cauchy-Schwarz inequality, and the fact $\Vert r_t \Vert_\infty \leq 1$ for $t=1,...,T$, it holds that 
\begin{align*}
&\max_{\mu \in \Delta^\Phi_{M,\delta}} \sum_{t=1}^T \langle \mu , r_t \rangle - \sum_{t=1}^T  \langle \mu^{\Phi \tilde{\theta}_t} , r_t \rangle \leq \max_{\mu \in \Delta^\Phi_{M,\delta}} \sum_{t=1}^T \langle \mu , r_t \rangle - \sum_{t=1}^T  \langle \mu^{\Phi \theta_t^*} , r_t \rangle + \sum_{t=1}^T \Vert \mu^{\Phi \theta_t^*} -  \mu^{\Phi \tilde{\theta}_t}\Vert_1.
\end{align*}

Notice that
\begin{align*}
&\Vert \mu^{\Phi \tilde{\theta}_{t-i}} - \mu^{\Phi \tilde{\theta}_{t-(i+1)}} \Vert_1\\
\leq&  \Vert \mu^{\Phi \theta_{t-i}^*} - \mu^{\Phi \theta_{t-(i+1)}^*}  \Vert_1 + \Vert \mu^{\Phi \tilde{\theta}_{t-i}} -  \mu^{\Phi \theta_{t-i}^*} \Vert_1 + \Vert  \mu^{\Phi \tilde{\theta}_{t-(i+1)}} -  \mu^{\Phi \theta_{t-(i+1)}^*} \Vert_1.
\end{align*}
Therefore, we have
\begin{align*}
& \quad \max_{\pi \in \Pi^\Phi} R(\pi,T)\\
& \leq 2(2\tau + 2) + \left[  \max_{\mu \in \Delta^\Phi_{M,\delta}} \sum_{t=1}^T \langle \mu , r_t \rangle - \sum_{t=1}^T  \langle \mu^{\Phi \theta_t^*} , r_t \rangle \right] + \left[ \sum_{t=1}^T \sum_{i=0}^{t-i} e^{-\frac{i}{\tau}} \Vert \mu^{\Phi \theta_{t-i}^*} - \mu^{\Phi \theta_{t-(i+1)}^*} \Vert_1 \right] \\
& + \sum_{t=1}^T \Vert \mu^{\Phi \theta_t^*} -  \mu^{\Phi \tilde{\theta}_t}\Vert_1 + \sum_{t=1}^T \sum_{i=0}^{t-i} e^{-\frac{i}{\tau}} \left( \Vert \mu^{\Phi \tilde{\theta}_{t-i}} -  \mu^{\Phi \theta_{t-i}^*} \Vert_1 + \Vert  \mu^{\Phi \tilde{\theta}_{t-(i+1)}} -\mu^{\Phi \theta_{t-(i+1)}^*} \Vert_1 \right) \\
&\leq O \left( \tau + 4 \sqrt{\tau T} \ln(T) + \sqrt{\tau T} \ln(\vert S \vert \vert A \vert) + e^{-\sqrt{T}} T \vert S \vert \vert A \vert \right)\\
& + \sum_{t=1}^T \Vert \mu^{\Phi \theta_t^*} -  \mu^{\Phi \tilde{\theta}_t}\Vert_1 +\sum_{t=1}^T \sum_{i=0}^{t-i} e^{-\frac{i}{\tau}} \left( \Vert \mu^{\Phi \tilde{\theta}_{t-i}} -  \mu^{\Phi \theta_{t-i}^*} \Vert_1 + \Vert  \mu^{\Phi \tilde{\theta}_{t-(i+1)}} -  \mu^{\Phi \theta_{t-(i+1)}^*} \Vert_1 \right),
\end{align*}
where the second inequality follows from the proof of Theorem \ref{theorem:mdp_rftl} since we chose the same parameters $\eta = \sqrt{\frac{T}{\tau}}, \delta = e^{-\sqrt{T}}$.

If we choose $K(t)$ such that $\Vert \mu^{\Phi \theta_t^*} -  \mu^{\Phi \tilde{\theta}_t}\Vert_1$ are less than or equal to a constant $\epsilon(\epsilon'_t, \epsilon''_t, \epsilon'''_t,K(t)) $ for all $t=1,...,T$ we have
\begin{align*}
& \sum_{t=1}^T \Vert \mu^{\Phi \theta_t^*} -  \mu^{\Phi \tilde{\theta}_t}\Vert_1 + \sum_{t=1}^T \sum_{i=0}^{t-i} e^{-\frac{i}{\tau}} \left( \Vert \mu^{\Phi \tilde{\theta}_{t-i}} -  \mu^{\Phi \theta_{t-i}^*} \Vert_1 + \Vert  \mu^{\Phi \tilde{\theta}_{t-(i+1)}} -  \mu^{\Phi \theta_{t-(i+1)}^*} \Vert_1 \right) \\
&\leq T {\epsilon} + 2 T  {\epsilon} (1 + \int_0^\infty e^{-\frac{x}{\tau} dx})\\
&\leq  T {\epsilon} + 2 T  {\epsilon} (1 + \tau)\\
&= T (1+2(1+\tau) ) {\epsilon}.
\end{align*}
We have that 
\begin{align*}
\max_{\pi \in \Pi^\Phi} R(\pi,T) \leq O \left( \tau + 4 \sqrt{\tau T} \ln(T) + \sqrt{\tau T} \ln(\vert S \vert \vert A \vert) + e^{-\sqrt{T}} T \vert S \vert \vert A \vert + T \tau {\epsilon} \right). 
\end{align*}

Let $\epsilon'_t = \epsilon''_t = \frac{1}{H_t}\left[ b_{K(t)} + t [2 + \frac{1}{\eta}(1 + \ln(\vert S \vert \vert A \vert ))]\right]$, $\epsilon'''_t = b_{K(t)}$. By Lemma \ref{dist_mu_star_mu_tilde}, we have that 
\begin{align*}
err \leq \tau \ln(\frac{1}{\epsilon'})(2\epsilon' + \epsilon'') + 3\epsilon' + c(\epsilon' + \epsilon'') + \sqrt{\frac{2 \eta}{t } (\epsilon''' + G_{F_t} c (\epsilon' + \epsilon''))}.
\end{align*}

By Lemma \ref{bound_G_F_t} we know that $G_{F_t}\leq t(1+2\sqrt{\tau} d W)$ so that 
\begin{align*}
{\epsilon} \leq \tau \ln(\frac{1}{\epsilon'})(2\epsilon' + \epsilon'') + 3\epsilon' + c(\epsilon' + \epsilon'') + \sqrt{\frac{2 \sqrt{T}}{\sqrt{\tau} } (\epsilon''' +  c[1+2\sqrt{\tau} d W] (\epsilon' + \epsilon''))},
\end{align*}
where we plugged in the value for $\eta$. It is easy to see that the right hand side of the last inequality bounded above by $O(\tau \ln(\frac{1}{\epsilon'})T^{1/4}c\sqrt{dW}(\epsilon'+\epsilon''+\epsilon'''))$. So that forcing all $\epsilon',\epsilon'',\epsilon'''$ to be $O(\frac{1}{\sqrt{dW}\tau^{3/2} T^{3/4}})$ will ensure $T\tau {\epsilon}$ to be $O( c \sqrt{\tau T})$ ensuring that $\max_{\pi \in \Pi^\Phi} R(\pi,T) \leq O( c \sqrt{\tau T} \ln(T) \ln(\vert S \vert \vert A \vert) )$.

Since $\epsilon' = \epsilon'' = \frac{1}{H_t} b_{K(t)} + \frac{1}{H_t} t 2 + \frac{1}{H_t} t \frac{\sqrt{\tau}}{\sqrt{T}}(1 + \ln(\vert S \vert \vert A \vert ))$ we choose $H_t = \sqrt{dW} t \tau^2 T^{3/4}$, this ensures that $\frac{1}{H_t} t 2 + \frac{1}{H_t} t \frac{\sqrt{\tau}}{\sqrt{T}}(1 + \ln(\vert S \vert \vert A \vert ))$ are bounded above by $O(\frac{1}{\sqrt{dW} \tau^{3/2} T^{3/4}})$. 
We now must choose $K(t)$ so that $\frac{1}{H_t}b_{K(t)}$ and $\epsilon'''_t$ are both $O(\frac{1}{\sqrt{dW} \tau^{3/2} T^{3/4}})$. Since by the choice of $H_t$ we have $\frac{1}{H_t}b_{K(t)} \leq b_{K(t)}$ it suffices to bound $b_{K(t)}$. 

Set $\kappa = \frac{1}{T^2}$ in Lemma \ref{sgd_guarantee_on_function_c} and recall we are working conditioned on $\mathcal{E}$, we have that for all $t=1,...,T$   

\begin{align*}
b_{K(t)} &= \frac{W t \sqrt{d} + H_t (C_1 + C_2) + \frac{t}{\eta} (1 + \ln (Wd) + \vert \ln(\delta) \vert) C_1}{\sqrt{K(t)}}\\
&+ \sqrt{\frac{(1+4S^2 K(t))(2\ln(\frac{1}{\kappa}) + d \ln(1 + \frac{W^2 K(t)}{d}))}{K(t)^2}}\\
& \leq O(\frac{Wt\sqrt{d}H_t (C_1 + C_2) \sqrt{T} \sqrt{\tau}\ln(WTd)}{\sqrt{T} \sqrt{K(t)}})\\
& = O(\frac{Wt^2 \sqrt{d} (C_1 + C_2) \tau^{5/2} T^{3/4} \ln(WTd)}{\sqrt{K(t)}}).
\end{align*}

Setting 
\begin{align*}
\frac{Wt^2 \sqrt{d} (C_1 + C_2) \tau^{5/2} T^{3/4} \ln(WTd)}{\sqrt{K(t)}} = \frac{1}{\sqrt{dW}\tau^{3/2} T^{3/4}}
\end{align*}
and solving for $K(t)$, we get that $K(t) = \left[ W^{3/2} t^2 d \tau^4 (C_1 + C_2) T^{3/2}\ln(W Td) \right]^2$, which ensures $b_{K(t)} = O(\frac{1}{\sqrt{dW}\tau^{3/2} T^{3/4}})$.

By the choice of $\kappa$ in Lemma \ref{sgd_guarantee_on_function_c}, we have that for each $t=1,...,T$ with probability at least $1-\frac{1}{T^2}$, $ \Vert \mu^{\Phi \theta_t^*} -  \mu^{\Phi \tilde{\theta}_t}\Vert_1 \leq O(\sqrt{dW}\frac{1}{\tau^{3/2} T^{3/4}})$.
This implies that 
\begin{align*}
&\text{$\Phi$-MDP-Regret}\\
&\leq  O( c \sqrt{\tau T} \ln(T) \ln(\vert S \vert \vert A \vert) )  P(\mathcal{E})\\
& + \left[
 O \left( \tau + 4 \sqrt{\tau T} \ln(T) + \sqrt{\tau T} \ln(\vert S \vert \vert A \vert) + e^{-\sqrt{T}} T \vert S \vert \vert A \vert + \sum_{t=1}^T \Vert \mu^{\Phi \theta_t^*} -  \mu^{\Phi \tilde{\theta}_t}\Vert_1 \right) \right] P(\mathcal{E}^c)
\end{align*}
Notice that since $ \mu^{\Phi \theta_t^*}$, and $\mu^{\Phi \tilde{\theta}_t}$ are probability distributions then $\Vert \mu^{\Phi \theta_t^*} - \mu^{\Phi \tilde{\theta}_t} \Vert_1 \leq 2$. So that 
\begin{align*}
 &\text{$\Phi$-MDP-Regret}  \leq  O( c \sqrt{\tau T} \ln(T) \ln(\vert S \vert \vert A \vert) ) + O(T) P(\mathcal{E}^c)
\end{align*}
where we upper bounded $P(\mathcal{E})$ with $1$. Notice that by the choice of $\kappa$, $P(\mathcal{E}^c) = P( \cup_{t=1}^T \mathcal{E}_i^c ) \leq \sum_{t=1}^T P(\mathcal{E}_i^c) \leq \frac{1}{T}$ so that $O(T) P(\mathcal{E}^c) = O(1)$. This completes the proof.

\end{proof}

\end{appendix}

\end{document}